\documentclass[letterpaper]{article}

\usepackage{times}
\usepackage[dvipdfmx]{graphicx}

\usepackage{natbib}

\usepackage{algorithm}
\usepackage{algorithmic}

\usepackage[accepted]{icml2018}

\usepackage{listings}
\usepackage{amsmath}
\usepackage{amssymb}
\usepackage{amsthm}
\usepackage{mathtools}
\usepackage{algorithm}
\usepackage{algorithmic}
\usepackage{here}
\usepackage{comment}
\usepackage{subfig}

\usepackage{enumitem}
\usepackage{relsize}
\setlist{nosep}

\newcounter{long}
\def\vector#1{\mbox{\boldmath $#1$}}
\DeclareMathOperator*{\argmin}{arg\,min}
\DeclareMathOperator*{\argmax}{arg\,max}
\DeclareMathOperator*{\argsup}{arg\,sup}

\theoremstyle{plain}
\newtheorem{theorem}{Theorem}
\newtheorem{lem}{Lemma}

\newtheorem{prop}{Proposition}

\newtheorem{rem}{Remark}
\theoremstyle{definition}
\newtheorem{definition}{Definition}
\theoremstyle{remark}

\icmltitlerunning{Does Distributionally Robust Supervised Learning Give Robust Classifiers?}

\begin{document}

\twocolumn[
\icmltitle{Does Distributionally Robust Supervised Learning Give Robust Classifiers?}



\icmlsetsymbol{equal}{*}

\begin{icmlauthorlist}
\icmlauthor{Weihua Hu}{tokyo,riken}
\icmlauthor{Gang Niu}{riken}
\icmlauthor{Issei Sato}{tokyo,riken}
\icmlauthor{Masashi Sugiyama}{riken,tokyo}
\end{icmlauthorlist}

\icmlaffiliation{tokyo}{University of Tokyo, Japan}
\icmlaffiliation{riken}{RIKEN, Tokyo, Japan}

\icmlcorrespondingauthor{Weihua Hu}{weihua916@gmail.com}

\icmlkeywords{supervised learning}

\vskip 0.3in
]



\printAffiliationsAndNotice{}  

\begin{abstract}
Distributionally Robust Supervised Learning (DRSL) is necessary for building reliable machine learning systems. 
When machine learning is deployed in the real world, its performance can be significantly degraded because test data may follow a different distribution from training data. 
DRSL with $f$-divergences explicitly considers the worst-case distribution shift by minimizing the adversarially reweighted training loss.
In this paper, we analyze this DRSL, focusing on the classification scenario.
Since the DRSL is explicitly formulated for a distribution shift scenario, we naturally expect it to give a robust classifier that can aggressively handle \emph{shifted distributions.}
However, surprisingly, we prove that the DRSL just ends up giving a classifier that exactly fits \emph{the given training distribution}, which is too pessimistic.
This pessimism comes from two sources: the particular losses used in classification and the fact that the variety of distributions to which the DRSL tries to be robust is too wide.
Motivated by our analysis, we propose simple DRSL that overcomes this pessimism and empirically demonstrate its effectiveness.
\end{abstract}
\section{Introduction}
Supervised learning has been successful in many application fields. 
The vast majority of supervised learning research falls into the Empirical Risk Minimization (ERM) framework \citep{vapnik1998statistical} that assumes a test distribution to be the same as a training distribution. 
However, such an assumption can be easily contradicted in real-world applications due to sample selection bias or non-stationarity of the environment \citep{quionero2009dataset}. 
Once the distribution shift occurs, the performance of the traditional machine learning techniques can be significantly degraded. This makes the traditional techniques unreliable for practitioners to use in the real world.

Distributionally Robust Supervised Learning (DRSL) is a promising paradigm to tackle this problem by obtaining prediction functions explicitly robust to distribution shift. More specifically, DRSL considers a minimax game between a learner and an adversary: the adversary first shifts the test distribution from the training distribution within a pre-specified uncertainty set so as to maximize the expected loss on the test distribution. The learner then minimizes the adversarial expected loss. 

DRSL with $f$-divergences \citep{bagnell2005robust, ben2009robust, duchi2016statistics, namkoong2016stochastic, namkoong2017variance} is particularly well-studied and lets the uncertainty set for test distributions be an $f$-divergence ball from a training distribution (see Section \ref{sec:prev_method} for the detail). 
This DRSL has been mainly studied under the assumption that \emph{the same continuous loss is used for training and testing}.
This is not the case in the \emph{classification} scenario, in which we care about the \emph{0-1 loss} (i.e., the mis-classification rate) at test time, while at training time, we use a \emph{surrogate loss} for optimization tractability.

In this paper, we revisit DRSL with $f$-divergences, providing novel insight for the \emph{classification} scenario.
In particular, we prove rather surprising results (Theorems \ref{lem:monotonic}--\ref{thm:steeper}), showing that when the DRSL is applied to classification, the obtained classifier ends up being optimal for the \emph{training distribution}. This is too pessimistic for DRSL given that DRSL is explicitly formulated for a distribution shift scenario and is naturally expected to give a classifier different from the one that exactly fits the given training distribution. Such pessimism comes from two sources: the particular losses used in classification and the over-flexibility of the uncertainty set used by DRSL with $f$-divergences.

Motivated by our analysis, we propose simple DRSL that overcomes the pessimism of the previous DRSL by incorporating structural assumptions on distribution shift (Section \ref{sec:proposed}).
We establish convergence properties of our proposed DRSL (Theorem \ref{thm:conv-rate-est-err-informal}) and derive efficient optimization algorithms (Section \ref{sec:opt}).
Finally, we demonstrate the effectiveness of our DRSL through experiments (Section \ref{sec:realdata}).
\ifnum\value{long}<1
All the appendices of this paper are provided in the supplementary material. 
\else
\fi
\vspace{-0.3cm}
\paragraph{Related work:} 
Besides DRSL with $f$-divergences, different DRSL considers different classes of uncertainty sets for test distributions. DRSL by \citet{nightmare} considered the uncertainty of features deletion, while DRSL by \citet{anqi14} considered the uncertainty of unknown properties of the conditional label distribution. 
DRSL by \citet{esfahani2015data}, \citet{blanchet2016robust} and \citet{sinha2017certifiable} lets the uncertainty set of test distributions be a Wasserstein ball from the training distribution. DRSL with the Wasserstein distance can make classifiers robust to adversarial examples \citep{sinha2017certifiable}, while DRSL with $f$-divergences can make classifiers robust against adversarial reweighting of data points as shown in Section \ref{sec:prev_method}. Recently, in the context of fair machine learning, \citet{hashimoto2018fairness} applied DRSL with $f$-divergences in an attempt to achieve fairness without demographic information.

\vspace{-0.3cm}
\section{Review of ERM and DRSL} \label{sec:prev_method}
In this section, we first review the ordinary ERM framework. Then, we explain a general formulation of DRSL and review DRSL with $f$-divergences.

Suppose training samples, $\{ (x_1, y_1), \ldots, (x_N, y_N)\} \equiv \mathcal{D}$, are drawn i.i.d.~from an unknown training distribution over $\mathcal{X}\times\mathcal{Y}$ with density $p(x,y),$ where $\mathcal{X} \subset \mathbb{R}^d$ and $\mathcal{Y}$ is an output domain.
Let $g_{\theta}$ be a prediction function with parameter $\theta$, mapping $x \in \mathcal{X}$ into a real scaler or vector, and let $\ell(\widehat{y},y)$ be a loss between $y$ and real-valued prediction $\widehat{y}$. 
\vspace{-0.2cm}
\paragraph{ERM:} 
The objective of the risk minimization (RM) is
\begin{align}
\min_{\theta} \underbrace{\mathbb{E}_{p(x, y)} [\ell(g_{\theta}(x), y)]}_{\qquad \mathlarger{\equiv \mathcal{R}(\theta)}}, \label{eq:rm}
\end{align}
where $\mathcal{R}(\theta)$ is called the risk.
In ERM, we approximate the expectation in Eq.~\eqref{eq:rm} by training data $\mathcal{D}$:
\begin{align}
\min_{\theta} \underbrace{\frac{1}{N} \sum_{i = 1}^N \ell(g_{\theta}(x_i), y_i)}_{\qquad \mathlarger{\equiv \widehat{\mathcal{R}}(\theta)}}, \label{eq:erm} 
\end{align}
where $\widehat{\mathcal{R}}(\theta)$ is called the empirical risk.
To prevent overfitting, we can add regularization term $\Omega(\theta)$ to Eq.~\eqref{eq:erm} and minimize $\widehat{\mathcal{R}}(\theta) + \lambda \Omega(\theta)$, where $\lambda \geq 0$ is a trade-off hyper-parameter.


\vspace{-0.2cm}
\paragraph{General formulation of DRSL:} 
ERM implicitly assumes the test distribution to be the same as the training distribution, which does not hold in most real-world applications.
DRSL is explicitly formulated for a distribution shift scenario, where test density $q(x,y)$ is different from training density $p(x,y)$. 
Let $\mathcal{Q}_{p}$ be an uncertainty set for test distributions.
In DRSL, the learning objective is 
\begin{align}
\min_{\theta} \sup_{q \in \mathcal{Q}_p} \mathbb{E}_{q(x, y)} [\ell(g_{\theta}(x), y)].  \label{eq:robust_formulation}
\end{align}
We see that Eq.~\eqref{eq:robust_formulation} minimizes the risk w.r.t.~the \emph{worst-case test distribution} within the uncertainty set $\mathcal{Q}_p$.

\vspace{-0.3cm}
\paragraph{DRSL with $f$-divergences:} 
Let $q \ll p$ denote that $q$ is absolutely continuous w.r.t.~$p$, i.e., $p(x, y) = 0$ implies $q(x,y)=0.$ \citet{bagnell2005robust} and \citet{ben2009robust} considered the particular uncertainty set
\begin{align} 
\mathcal{Q}_p = \{q \ll p \ |\  {\rm D}_f[q \| p]\leq \delta \} \label{eq:unc}, 
\end{align}
where ${\rm D}_f[\cdot \| \cdot]$ is an $f$-divergence defined as ${\rm D}_f[q \| p] \equiv \mathbb{E}_{p} \left[f \left(  q/p \right)\right]$, and $f(\cdot)$ is convex with $f(1) = 0$. The $f$-divergence \citep{ciszar1967information} measures a discrepancy between probability distributions. When $f(x) = x \log x$, we have the well-known Kullback-Leibler divergence as an instance of it.
Hyper-parameter $\delta > 0$ in Eq.~\eqref{eq:unc} controls the degree of the distribution shift.
Define $r(x,y) \equiv q(x,y) / p(x,y)$. Through some calculations, the objective of DRSL with $f$-divergences can be rewritten as
\vspace{-0.5cm}

{\small
\begin{align}
&\min_{\theta} \underbrace{\sup_{r \in \mathcal{U}_f} \mathbb{E}_{p(x,y)} [r(x, y) \ell(g_{\theta}(x), y)]}_{\qquad \qquad  \mathlarger{\equiv \mathcal{R}_{\rm adv}(\theta)}},  \label{eq:1234} \\
&\mathcal{U}_f \equiv \{ r(x,y)\  |\ \mathbb{E}_{p(x, y)} \left[ f \left( r(x,y) \right)\right] \leq \delta,\nonumber \\ 
& \qquad \qquad \ \ \ \qquad \mathbb{E}_{p(x,y)}[r(x,y)] = 1, \nonumber \\
&\qquad  \qquad \ \ \ \qquad  r(x, y) \geq 0, \ \forall (x, y) \in \mathcal{X} \times \mathcal{Y}\}. \label{eq:12345}
\end{align}}We call $\mathcal{R}_{\rm adv}(\theta)$ the \emph{adversarial risk} and call the minimization problem of Eq.~\eqref{eq:1234} the \emph{adversarial risk minimization} (ARM).
In ARM, the density ratio, $r(x,y)$, can be considered as the weight put by the adversary on the loss of data $(x,y)$.
Then, Eq.~\eqref{eq:1234} can be regarded as a minimax game between the learner (corresponding to $\min_{\theta}$) and the adversary (corresponding to $\sup_{r \in \mathcal{U}_f}$): the adversary first reweights the losses using $r(\cdot, \cdot)$ so as to maximize the expected loss; the learner then minimizes the \emph{reweighted} expected loss, i.e., adversarial risk $\mathcal{R}_{\rm adv}(\theta)$.

For notational convenience, we denote $\ell(g_{\theta}(x_i), y_i)$ by $\ell_i(\theta)$. Also, let $\vector{r} \equiv (r_1, \ldots, r_N)$ be a vector of density ratios evaluated at training data points, i.e., $r_i \equiv r(x_i, y_i)$ for $1 \leq i \leq N$.
Equations~\eqref{eq:1234} and \eqref{eq:12345} can be empirically approximated as\footnote{The formulation in Eqs.~\eqref{eq:uncertain_set_empirical} and \eqref{eq:uncertain_set_empirical2} is similar to \citet{duchi2016statistics}, \citet{namkoong2016stochastic} and \citet{namkoong2017variance} except that they decay $\delta$ linearly w.r.t.~the number of training data $N$. Different from us, they assume $\delta = 0$ in Eq.~\eqref{eq:unc} (thus, their objective is the ordinary risk) and try to be robust to apparent distribution fluctuations due to the \emph{finiteness of training samples}. On the other hand, we consider using the same $\delta > 0$ for both Eqs.~\eqref{eq:unc} and \eqref{eq:uncertain_set_empirical2} and try to be robust to the actual distribution change between training and test stages.}
{\small
\begin{align}
&\min_{\theta} \underbrace{\sup_{\vector{r} \in \widehat{\mathcal{U}}_{f}} \frac{1}{N}\sum_{i = 1}^{N} r_i \ell_i(\theta)}_{\qquad \qquad \mathlarger{\equiv \widehat{\mathcal{R}}_{\rm adv}(\theta)}}, 
\label{eq:uncertain_set_empirical} \\
&\widehat{\mathcal{U}}_{f}= \left\{ \vector{r} \   \middle|  \ \frac{1}{N}\sum_{i=1}^N f \left( r_i \right) \leq \delta,\  \frac{1}{N}\sum_{i=1}^N r_i = 1,\ \vector{r} \geq 0 \right \}, \label{eq:uncertain_set_empirical2}
\end{align}}where the inequality constraint for a vector is applied in an element-wise fashion. 
We call $\widehat{\mathcal{R}}_{\rm adv}(\theta)$ the \emph{adversarial empirical risk} and call the minimization problem of Eq.~\eqref{eq:uncertain_set_empirical} the \emph{adversarial empirical risk minimization} (AERM).
In AERM, the adversary (corresponding to $\sup_{\vector{r} \in \widehat{\mathcal{U}}_{f}}$) reweights data losses through $\vector{r}$ to maximize the empirical loss in Eq.~\eqref{eq:uncertain_set_empirical}.
To prevent overfitting, we can add regularization term $\Omega(\theta)$ to Eq.~\eqref{eq:uncertain_set_empirical}.
\vspace{-0.2cm}
\section{Analysis of DRSL with $f$-divergences in classification} \label{sec:analysis}
At first glance, DRSL with $f$-divergences (which we call ARM and AERM in this paper) is reasonable to give a distributionally robust classifier in the sense that it explicitly minimizes the loss for the shifted worst-case \emph{test distribution.}
However, we show rather surprising results, suggesting that the DRSL, when applied to \emph{classification}, still ends up giving a classifier optimal for a \emph{training} distribution. This is too pessimistic for DRSL because it ends up behaving similarly to ordinary ERM-based supervised classification that does \emph{not} explicitly consider distribution shift. To make a long story short, our results hold because of the particular losses used in classification (especially, the 0-1 loss at test time) and the overly flexible uncertainty sets used by ARM and AERM. We will detail these points after we state our main results. 
\vspace{-0.3cm}
\paragraph{Classification setting:}
Let us first briefly review classification settings to set up notations.
In binary classification, we have $g_{\theta}(\cdot): x \mapsto \widehat{y} \in \mathbb{R}$, $\mathcal{Y} = \{ +1,-1 \}$ and $\ell(\cdot, \cdot): \mathbb{R} \times \mathcal{Y} \to \mathbb{R}_{\geq 0}$. In $K$-class classification for $K \geq 2$, we have $g_{\theta}(\cdot): x \mapsto \widehat{y} \in \mathbb{R}^{K}$, $\mathcal{Y} = \{1, 2, \ldots, K\}$ and $\ell(\cdot, \cdot): \mathbb{R}^K \times \mathcal{Y} \to \mathbb{R}_{\geq 0}$.
The goal of classification is to learn the prediction function that minimizes the mis-classification rate on the test distribution. The mis-classification rate corresponds to the \emph{use of the 0-1 loss}, i.e., $\ell(\widehat{y}, y) \equiv \vector{1} \{{\rm sign}(\widehat{y}) \neq y\}$ for binary classification, and $\ell(\widehat{y}, y) \equiv \vector{1} \{{\rm argmax}_k \widehat{y}_k \neq y\}$ for multi-class classification, where $\vector{1}\{\cdot \}$ is the indicator function and $\widehat{y}_k$ is the $k$-th element of $\widehat{y} \in \mathbb{R}^{K}$.
However, since the 0-1 loss is non-convex and non-continuous, learning with it is difficult in practice. Therefore, at training time, we instead use \emph{surrogate losses} that are easy to optimize, such as the logistic loss and the cross-entropy loss. 

In the following, we state our main results, analyzing ARM and AERM in the classification scenario by considering the use of the 0-1 loss and a surrogate loss. 
\vspace{-0.3cm}
\paragraph{The 0-1 loss case:}
Theorem \ref{lem:monotonic} establishes the non-trivial relationship between the adversarial risk and the ordinary risk when the 0-1 loss is used.
\begin{theorem} \label{lem:monotonic}
Let $\ell(\widehat{y}, y)$ be the 0-1 loss. Then, there is a monotonic relationship between $\mathcal{R}_{\rm adv}(\theta)$ and $\mathcal{R}(\theta)$ in the sense that for any pair of parameters $\theta_1$ and $\theta_2$, the followings hold.

If $\mathcal{R}_{\rm adv}(\theta_1) < 1$, then
\begin{align} 
\mathcal{R}_{\rm adv}(\theta_1) < \mathcal{R}_{\rm adv}(\theta_2) \Longleftrightarrow \mathcal{R}(\theta_1) <  \mathcal{R}(\theta_2). \label{eq:monotonic_erm}
\end{align}

If $\mathcal{R}_{\rm adv}(\theta_1)= 1$, then
\begin{align} 
\mathcal{R}(\theta_1) \leq \mathcal{R}(\theta_2) \Longrightarrow \mathcal{R}_{\rm adv}(\theta_2) = 1.  \label{eq:monotonic_erm2}
\end{align}
The same monotonic relationship also holds between their empirical approximations: $\widehat{\mathcal{R}}_{\rm adv}(\theta)$ and $\widehat{\mathcal{R}}(\theta).$
\end{theorem}
\vspace{-0.2cm}
See Appendix \ref{app:proof_monotonic} for the proof. 
Theorem \ref{lem:monotonic} shows a surprising result that \emph{when the 0-1 loss is used}, $\mathcal{R}(\theta)$ and $\mathcal{R}_{\rm adv}(\theta)$ are essentially equivalent objective functions in the sense that the minimization of one objective function results in the minimization of another objective function. This readily implies that $\mathcal{R}(\theta)$ and $\mathcal{R}_{\rm adv}(\theta)$ have exactly the same set of global minima in the regime of $\mathcal{R}_{\rm adv}(\theta) < 1$. An immediate practical implication is that if we select hyper-parameters such as $\lambda$ for regularization according to \emph{the adversarial risk with the 0-1 loss}, we will end up choosing hyper-parameters that attain the minimum mis-classification rate on the \emph{training distribution}. 
\vspace{-0.3cm}
\paragraph{The surrogate loss case:}
We now turn our focus on the training stage of classification, where we use a \emph{surrogate loss} instead of the 0-1 loss.
In particular, for binary classification, we consider a class of classification calibrated losses \citep{bartlett2006convexity} that are margin-based, i.e., $\ell(\widehat{y}, y)$ is a function of product $y \widehat{y}$. 
For multi-class classification, we consider a class of classification calibrated losses \citep{tewari2007consistency} that are invariant to class permutation, i.e., for any class permutation $\pi: \mathcal{Y} \to \mathcal{Y}$,  $\ell(\widehat{y}^{\pi}, \pi(y))$ = $\ell(\widehat{y}, y)$ holds, where $\widehat{y}^{\pi}_k = \widehat{y}_{\pi(k)}$ for $1 \leq k \leq K$. 
Although we only consider the sub-class of general classification-calibrated losses \cite{bartlett2006convexity, tewari2007consistency}, we note that ours still includes some of the most widely used losses: the logistic, hinge, and exponential losses for binary classification and the softmax cross entropy loss for multi-class classification.

We first review Proposition \ref{thm:calibrated_review} by \citet{bartlett2006convexity} and \citet{tewari2007consistency} that justifies the use of classification-calibrated losses in ERM for classification.
We then show a surprising fact in Theorem \ref{thm:calibrated} that the similar property also holds for ARM using the sub-class of classification-calibrated losses.
\begin{prop}[\citet{bartlett2006convexity, tewari2007consistency}]  \label{thm:calibrated_review}
Let $\ell(\widehat{y}, y)$ be a classification calibrated loss, and assume that the hypothesis class is equal to all measurable functions. 
Then, the risk minimization (RM) gives the Bayes optimal classifier\footnote{The classifier that minimizes the mis-classification rate for the training density $p(x, y)$ (the 0-1 loss is considered), i.e., the classifier whose prediction on $x$ is equal to $\argmax_{y \in \mathcal{Y}} p(y | x)$.}.
\end{prop}

\begin{theorem} \label{thm:calibrated}
Let $f(\cdot)$ be differentiable, the hypothesis class be all measurable functions, and $\ell(\widehat{y}, y)$ be a classification-calibrated loss that is margin-based or invariant to class permutation.
Let $g^{\rm (adv)}$ be any solution of ARM\footnote{There can be multiple solutions that achieve the same minimum adversarial risk.} under the above setting, and define
\begin{align}
r^{\ast} \equiv \argmax_{r \in \mathcal{U}_f} \mathbb{E}_{p(x,y)} [r(x, y) \ell(g^{\rm (adv)}(x), y)].
\end{align}
Then, the prediction of $g^{\rm (adv)}$ coincides with that of the Bayes optimal classifier almost surely over $q^{\ast}(x) \equiv \sum_{y \in \mathcal{Y}}r^{\ast}(x, y) p(x, y)$.
Furthermore, among the solutions of ARM, there exists $g^{\rm (adv)}$ whose prediction coincides with that of the Bayes optimal classifier almost surely over $p(x)$.
\end{theorem}
See Appendix \ref{app:proof_calibrated} for the proof.
Theorem \ref{thm:calibrated} indicates that ARM, similarly to RM, ends up giving the optimal decision boundary for the \emph{training} distribution, if the hypothesis class is all measurable functions and we have access to true density $p(x, y)$.
Even though the assumptions made are strong, Theorem \ref{thm:calibrated} together with Proposition \ref{thm:calibrated_review} highlight the non-trivial fact that when a certain surrogate loss is used, AERM and ERM demonstrate the similar \emph{asymptotic} behavior in classification.

We proceed to consider a more practical scenario, where we only have a finite amount of training data and the hypothesis class is limited.
In the rest of the section, we focus on a differentiable loss and a real-valued scalar output, i.e., $\widehat{y} \in \mathbb{R}$, which includes the scenario of binary classification.

We first define the notion of a \emph{steeper} loss, which will play a central role in our result.

\begin{definition}[Steeper loss] \label{def:steeper}
Loss function $\ell_{\rm steep}(\widehat{y}, y)$ is said to be steeper than loss function $\ell(\widehat{y}, y)$, if there exists a non-constant, non-decreasing and non-negative function $h: \mathbb{R}_{\geq 0} \to \mathbb{R}_{\geq 0}$ such that
\begin{align}
 \frac{\partial \ell_{\rm steep}(\widehat{y}, y)}{\partial \widehat{y}} =  h(\ell(\widehat{y}, y)) \frac{\partial \ell(\widehat{y}, y)}{\partial \widehat{y}}. \label{eq:hihi0} 
\end{align}
\end{definition}

For example, following Definition \ref{def:steeper}, we can show that the exponential loss is \emph{steeper} than the logistic loss. 
Intuitively, outlier-sensitive losses are \emph{steeper} than more outlier-robust losses.
Lemma \ref{lem:steeper_calibrated} shows an important property of a steeper loss in a classification scenario. 
\begin{lem} \label{lem:steeper_calibrated}
Let $\ell(\widehat{y}, y)$ be a margin-based convex classification-calibrated loss. Then, its steeper loss defined in Eq.~\eqref{eq:hihi0} is also convex classification-calibrated if $h(\ell(0, y)) > 0$.
\end{lem}
See Appendix \ref{app:proof_steeper_calibrated} for the proof.

Now we are ready to state our result in Theorem \ref{thm:steeper} that considers $\widehat{y} \in \mathbb{R}$. Theorem \ref{thm:steeper} holds for \emph{any} hypothesis class that is parametrized by $\theta$ and sub-differentiable w.r.t.~$\theta$, e.g., linear-in-parameter models and deep neural networks.

\begin{theorem} \label{thm:steeper}
Let $\theta^{\ast}$ be a stationary point of AERM in Eq.~\eqref{eq:uncertain_set_empirical} using $\ell(\widehat{y}, y)$.
Then, there exists a steeper loss function, $\ell_{\rm DRSL}(\widehat{y}, y)$, such that $\theta^{\ast}$ is also a stationary point of the following ERM.
\begin{align}
\min_{\theta} \frac{1}{N}\sum_{i = 1}^N \ell_{{\rm DRSL}}(g_{\theta}(x_i), y_i). \label{eq:reg_erm}
\end{align}
\end{theorem}
See Appendix \ref{app:proof_steeper} for the proof.
\begin{rem}[Conditions for convexity] \label{rem:convex}
\upshape
Let $\ell(\widehat{y}, y)$ be convex in $\widehat{y}$, $g_{\theta}(x)$ be a linear-in-parameter model. 
Then, both AERM in Eq.~\eqref{eq:uncertain_set_empirical} and ERM in Eq.~\eqref{eq:reg_erm} become convex in $\theta$.
This implies that the stationary point $\theta^{\ast}$ in Theorem \ref{thm:steeper} turns out to be the \emph{global optimum} for both Eqs.~\eqref{eq:uncertain_set_empirical} and \eqref{eq:reg_erm} in this usual setting.
\end{rem}
\vspace{-0.2cm}
Note that Theorem \ref{thm:steeper} holds for general real-valued scalar prediction, i.e., $\widehat{y} \in \mathbb{R}$; thus, the result holds for ordinary regression (using the same loss for training and testing) as well as for binary classification. However, as we discuss in the following, the implication of Theorem \ref{thm:steeper} is drastically different for the two scenarios.

{\bf Implication for classification:}
Theorem \ref{thm:steeper} together with Lemma \ref{lem:steeper_calibrated} indicate that under a mild condition,\footnote{The condition that $h(\ell(0, y))>0$ in Lemma \ref{lem:steeper_calibrated}. Whether the condition holds or not generally depends on the uncertainty set, the model, the loss function, and training data. Nonetheless, the condition is mild in practice; especially, the condition always holds when the Kullback-Leibler divergence is used. See Appendix \ref{app:proof_steeper_calibrated} for detailed discussion.} AERM using a convex classification-calibrated margin-based loss reduces to Eq.~\eqref{eq:reg_erm}, which is ERM using a convex classification-calibrated loss. 
This implies that AERM, similarly to ordinary ERM using a classification-calibrated loss, will try to give a classifier optimal for the \emph{training distribution}. 

{\bf \emph{Why does the use of the steeper surrogate loss fail to give meaningful robust classifiers?}}
This is because we are dealing with classification tasks, where we care about the performance \emph{in terms of the 0-1 loss} at test time. The use of the steeper \emph{surrogate} loss may make a classifier distributionally robust \emph{in terms of the surrogate loss},\footnote{For fixed $\delta$ (non-decaying w.r.t.~$N$), whether AERM is consistent with ARM or not is an open problem. Nevertheless, we empirically confirm in Section \ref{sec:realdata} that AERM achieves lower adversarial risk than other baselines \emph{in terms of the surrogate loss}.} but not necessarily so in terms of the 0-1 loss. 
Moreover, even if we obtain a classifier that minimizes the adversarial risk \emph{in terms of the 0-1 loss}, the obtained classifier ends up being optimal for the training distribution (see Theorem \ref{lem:monotonic}).
In any case, the use of the steeper loss does not in general give classifiers that are robust to change from a training distribution.

In summary, in the \emph{classification} scenario, the use of the steeper loss does more harm (making a classifier sensitive to outliers due to the use of the steeper surrogate loss) than good (making a classifier robust to change from a training distribution).

{\bf Implication for ordinary regression:}
For comparison, let us rethink about the classical regression scenario, in which we use \emph{the same loss}, e.g., the squared loss, during training and testing. 
In such a case, the use of the steeper loss may indeed make regressors distributionally robust \emph{in terms of the same loss}. 
Nonetheless, learning can be extremely sensitive to outliers due to the use of the \emph{steeper} loss. 
Hence, when applying DRSL with $f$-divergences to real-world regression tasks, we need to pay extra attention to ensure that there are no outliers in datasets.


\section{DRSL with Latent Prior Probability Change} \label{sec:proposed}
In this section, motivated by our theoretical analysis in Section \ref{sec:analysis}, we propose simple yet practical DRSL that overcomes the over pessimism of ARM and AERM in the classification scenario. We then analyze its convergence property and discuss the practical use of our DRSL.

\paragraph{Theoretical motivation:}
What insight can we get from our theoretical analyses in Section \ref{sec:analysis}?
Our key insight from proving the theorems is that the adversary of ARM has too much (non-parametric) freedom to shift the test distribution, and as a result, the learner becomes overly pessimistic.
In fact, the proofs of all the theorems rely on the over-flexibility of the uncertainty set $\mathcal{U}_f$ in Eq.~\eqref{eq:12345}, i.e., the values of $r(\cdot, \cdot)$ are \emph{not tied together} for different $(x,y)$ within $\mathcal{U}_f$ (see Eqs.~\eqref{eq:1234} and \eqref{eq:12345}).
Consequently, the adversary of ARM simply assigns larger weight $r(x, y)$ to data $(x, y)$ with a larger loss.
This fact, combined with the fact that we use the different losses during training and testing in classification (see discussion at the end of Section \ref{sec:analysis}), led to the pessimistic results of Theorems \ref{lem:monotonic}--\ref{thm:steeper}.

Our theoretical insight suggests that in order to overcome the pessimism of ARM applied to classification, it is crucial to \emph{structurally constrain} $r(\cdot, \cdot)$ in $\mathcal{U}_f$, or equivalently, to impose \emph{structural assumptions} on the distribution shift.
To this end, in this section, we propose DRSL that overcomes the limitation of the DRSL by incorporating structural assumptions on distribution shift.
 
\paragraph{Practical structural assumptions:}
In practice, there can be a variety of ways to impose structural assumptions on distribution shift. 
Here, as one possible way, we adopt the \emph{latent prior probability change assumption} \citep{storkey2007mixture} 
because this particular class of assumptions enjoys the following two practical advantages. 
\begin{enumerate}
\setlength{\parskip}{0cm}
  \setlength{\itemsep}{0cm}
\item Within the class, users of our DRSL can easily and intuitively model their assumptions on distribution shift (see the discussion at the end of this section).
\item Efficient learning algorithms can be derived (see Section \ref{sec:opt}).
\end{enumerate}
Let us introduce a latent variable $z \in \mathcal{Z} \equiv \{ 1, \ldots, S \}$, which we call a \emph{latent category}, where $S$ is a constant. The latent prior probability change assumes
\begin{align}
p(x, y | z) = q(x, y| z),~~~ q(z) \neq p(z), \label{eq:assumption}
\end{align}
where $p$ and $q$ are the training and test densities, respectively.
The intuition is that
we assume a two-level \emph{hierarchical} data-generation process: we first sample latent category $z$ from the prior and then sample actual data $(x,y)$ conditioned on $z$.
We then assume that only the prior distribution over the latent categories changes, leaving the conditional distribution intact.

We assume the structural assumption in Eq.~\eqref{eq:assumption} to be provided externally by users of our DRSL based on their knowledge of potential distribution shift, rather than something to be inferred from data. As we will see at the end of this section, specifying Eq.~\eqref{eq:assumption} amounts to grouping training data points according to their latent categories, which is quite intuitive to do in practice.

\paragraph{Objective function of our DRSL:}
With the latent prior probability change of Eq.~\eqref{eq:assumption}, the uncertainty set for test distributions in our DRSL becomes
\begin{align}
\mathcal{Q}_p = \{q \ll p \  |\  &{\rm D}_f [q(x, y, z) || p(x, y, z)] \leq \delta,\ \nonumber \\
& \  q(x, y | z) = p(x, y| z) \}. \label{eq:qp_for_grouping}
\end{align}

Then, corresponding to Eq.~\eqref{eq:robust_formulation}, the objective of our DRSL can be written as
{\small
\begin{align}
&\min_{\theta} \underbrace{\sup_{w \in \mathcal{W}_f}   \mathbb{E}_{p(x, y, z)} \left[ w(z) \ell(g_{\theta}(x), y) \right]}_{\qquad \qquad \mathlarger{\equiv \mathcal{R}_{\rm s\mathchar`-adv}(\theta)}},  \label{eq:re-setting2} \\
&\mathcal{W}_f \equiv \Bigg\{ w(z)\  \Bigg|\  \sum_{z \in \mathcal{Z}} p(z) f \left( w(z) \right) \leq \delta, \nonumber \\ 
& \qquad \qquad \sum_{z \in \mathcal{Z}} p(z)w(z) = 1, \ w(z) \geq 0,\ \forall z \in \mathcal{Z} \Bigg\}, \label{eq:uncertain_set2}
\end{align}}where $w(z) \equiv q(z)/p(z) = q(x,y,z)/p(x,y,z)$ because of $q(x, y | z) = p(x, y| z)$. 
We call $\mathcal{R}_{\rm s\mathchar`-adv}(\theta)$ the \emph{structural adversarial risk} and
call the minimization problem of Eq.~\eqref{eq:re-setting2} the \emph{structural adversarial risk minimization} (structural ARM).
Similarly to ARM, structural ARM is a minimax game between the learner and the adversary. Differently from ARM, the adversary of structural ARM (corresponding to $\sup_{w \in \mathcal{W}_f}$) uses $w(\cdot)$ to reweight data; hence, it has much less (only parametric) freedom to shift the test distribution compared to the adversary of ARM that uses non-parametric weight $r(\cdot, \cdot)$ (see Eq.~\eqref{eq:1234}).
Because of this limited freedom for the adversary, we can show that Theorems \ref{lem:monotonic}--\ref{thm:steeper} do \emph{not} hold for structural ARM, and we can expect to learn meaningful classifiers that are robust to \emph{structurally constrained} distribution shift.
\vspace{-0.2cm}
\paragraph{Discussion and proposal of evaluation metric for distributional robustness:}
Recall from Theorem \ref{lem:monotonic} that when the 0-1 loss is used, the adversarial risk ends up being equivalent to the ordinary risk as an evaluation metric, which is too pessimistic as a metric for the distributional robustness of a classifier.
In contrast, we can easily verify that our structural adversarial risk using the 0-1 loss does not suffer from the pessimism. 
We argue that our structural adversarial risk can be an alternative metric in distributionally robust classification.
To better understand its property, inspired by \citet{namkoong2017variance}, we decompose it as\footnote{This particular decomposition holds when the Pearson (PE) divergence is used and $\delta$ is not so large. Refer to Appendix \ref{sec:decomposition} for the derivation. Analogous decomposition can be also derived when other $f$-divergences are used.}
\vspace{-0.4cm}

{\scriptsize \begin{align}
\mathcal{R}_{\rm s\mathchar`-adv}(\theta) = \underbrace{\mathcal{R}(\theta)}_{\mathlarger{\text{(a) ordinary risk}}} +  \sqrt{\delta} \cdot \underbrace{\sqrt{\sum_{z \in \mathcal{Z}} p(z) (\mathcal{R}_z(\theta) - \mathcal{R}(\theta))^2}}_{\mathlarger{\text{(b) sensitivity}}}, \label{eq:decompose}
\end{align}}where $\mathcal{R}_z(\theta) (\equiv \mathbb{E}_{p(x, y|z)}[ \ell(g_{\theta}(x), y)]$) is the risk of the classifier on latent category $z \in \mathcal{Z}$.
We see that $\mathcal{R}_{\rm s\mathchar`-adv}(\theta)$ in Eq.~\eqref{eq:decompose} contains the \emph{risk variance} term (b). 
This variance term (b) can be large when the obtained classifier performs extremely poorly on a small number of latent categories.
Once a test density concentrates on those poorly-performed latent categories, the test accuracy of the classifier can extremely deteriorate.
In this sense, the classifier with large (b) is sensitive to distribution shift.
In contrast, the small risk variance (b) indicates that the obtained classifier attains similar accuracy on all the latent categories. In such a case, the test accuracy of the classifier is insensitive to latent category prior change. In this sense, the classifier with small (b) is robust to distribution shift.
To sum up, the additional term (b) measures the sensitivity of the classifier to the specified structural distribution shift.

On the basis of the above discussion, we see that $\mathcal{R}_{\rm s\mathchar`-adv}(\theta)$ in Eq.~\eqref{eq:decompose} simultaneously captures (a) the ordinary risk, i.e., the mis-classification rate \emph{when no distribution shift occurs}, and (b) the \emph{sensitivity} to distribution shift. 
In this sense, our structural adversarial risk is an intuitive and reasonable metric for distributional robustness of a classifier, and we will employ this metric in our experiments in Section \ref{sec:realdata}. 
\vspace{-0.2cm}
\paragraph{Empirical approximation:}
We explain how to empirically approximate the objective functions in Eqs.~\eqref{eq:re-setting2} and \eqref{eq:uncertain_set2} using training data $\mathcal{D}^{\prime} \equiv \{(x_1, y_1,z_1), \ldots, (x_N, y_N, z_N)\}$ drawn independently from $p(x, y, z)$. 

Define $\mathcal{G}_s \equiv \{i \ |\ z_i = s, 1 \leq i \leq N \}$ for $1\leq s \leq S$, which is a set of data indices belonging to latent category $s$. 
In our DRSL, users are responsible for specifying the groupings of training data points, i.e., $\{ \mathcal{G}_s \}_{s = 1}^{S}$.
By specifying these groupings, the users incorporate their structural assumptions on distribution shift into our DRSL.
We will discuss how this can be done in practice at the end of this section.

For notational convenience, let $w_s \equiv w(s),\ 1 \leq s \leq S$, and define $\vector{w} \equiv (w_1, \ldots, w_S)$. 
Equations~\eqref{eq:re-setting2} and \eqref{eq:uncertain_set2} can be empirically approximated as follows using $\mathcal{D}^{\prime}$:
\vspace{-0.2cm}
{\small \begin{align}
&\min_{\theta} \underbrace{\sup_{\vector{w} \in \widehat{\mathcal{W}}_{f}} \frac{1}{N}\sum_{s = 1}^{S} n_s w_s \widehat{\mathcal{R}}_s(\theta)}_{\qquad \qquad \mathlarger{\equiv  \widehat{\mathcal{R}}_{\rm s\mathchar`-adv}(\theta)}} \label{eq:general_opt_cluster}  \\ 
&\widehat{\mathcal{W}}_{f}= \Biggl\{ \vector{w}\in \mathbb{R}^{S} \Bigg| \frac{1}{N}\sum_{s=1}^S n_s f \left( w_s \right) \leq \delta, \nonumber \\
& \qquad \qquad \qquad \qquad \frac{1}{N}\sum_{s=1}^S n_s w_s = 1,\ \vector{w} \geq 0 \Biggr\}, \label{eq:uncertain_set_empirical_cluster}
\end{align}}where $n_s$ is the cardinality of $\mathcal{G}_s$ and $\widehat{\mathcal{R}}_s(\theta) (\equiv \frac{1}{n_s} \sum_{i \in \mathcal{G}_s} \ell_i(\theta))$ is the average loss of all data points in $\mathcal{G}_s$. 
We call $\widehat{\mathcal{R}}_{\rm s\mathchar`-adv}(\theta)$ the \emph{structural adversarial empirical risk}
and call the minimization problem of Eq.~\eqref{eq:general_opt_cluster} the \emph{structural adversarial empirical risk minimization} (structural AERM).
We can add regularization term $\Omega(\theta)$ to Eq.~\eqref{eq:general_opt_cluster} to prevent overfitting.
\vspace{-0.3cm}
\paragraph{Convergence rate and estimation error:}
We establish the convergence rate of the model parameter and the order of the estimation error for structural AERM in terms of the number of training data points $N$. 
Due to the limited space, we only present an informal statement here. The formal statement can be found in Appendix~\ref{sec:stat-conv-rate} and its proof can be found in Appendix~\ref{sec:proof-conv-rate}.

\begin{theorem}[Convergence rate and estimation error, informal statement] 
  \label{thm:conv-rate-est-err-informal}
Let $\theta^*$ be the solution of structural ARM, and $\widehat{\theta}_N$ be the solution of regularized structural AERM given training data of size $N$. Assume $g_{\theta}(x)$ is linear in $\theta$, and regularization hyper-parameter $\lambda$ decreases at a rate of $\mathcal{O}(N^{-1/2})$. Under mild conditions, as $N\to\infty$, we have $\|\widehat{\theta}_N-\theta^*\|_2=\mathcal{O}(N^{-1/4})$ and consequently, $| \mathcal{R}_{\rm s\mathchar`-adv}(\widehat{\theta}_N) - \mathcal{R}_{\rm s\mathchar`-adv}(\theta^{\ast}) | = \mathcal{O}(N^{-1/4})$.
\end{theorem}

Notice that the convergence rate of $\widehat{\theta}_N$ to $\theta^*$ is not the optimal parametric rate $\mathcal{O}(N^{-1/2})$. This is because the inner maximization of Eq.~\eqref{eq:general_opt_cluster} converges in $\mathcal{O}(N^{-1/4})$ that slows down the entire convergence rate. 
Theorem~\ref{thm:conv-rate-est-err-informal} applies to any $f$-divergence where $f(t)$ is nonlinear in $t$, while knowing which $f$-divergence is used may improve the result to the optimal parametric rate.
\vspace{-0.2cm}
\paragraph{Discussion on groupings: }
In our structural ARM and AERM, users need to incorporate their structural assumptions by grouping training data points. Here, we discuss how this can be done in practice.


Most straightforwardly, a user of our DRSL may assume class prior change \citep{saerens2002adjusting} or sub-category\footnote{A sub-category \citep{ristin2015categories} is a \emph{refined} category of a class label, e.g., a ``flu'' label contains three \emph{sub-categories}: types A, B, and C flu.} prior change.
To incorporate such assumptions into our DRSL, the user can simply group training data by class labels or a sub-categories, respectively. 

Alternatively, a user of our DRSL can group data by available meta-information of data such as time and places in which data are collected. 
The intuition is that data collected in the same situations (e.g., time and places) are likely to \emph{``share the same destiny''} in the future distribution shift; hence, it is natural to assume that only the \emph{prior} over the situations changes at test time while the conditionals remain the same. 

In any case, it is crucial that the users provide structural assumptions on distribution shift so that we can overcome the pessimism of ARM and AERM for classification raised in Section \ref{sec:analysis}.

\vspace{-0.3cm}
\section{Efficient Learning Algorithms} \label{sec:opt}
In this section, we derive efficient gradient-based learning algorithms for our structural AERM in Eq.~\eqref{eq:general_opt_cluster}.
Thanks to Danskin's theorem \citep{danskin67}, we can obtain the gradient $\nabla_{\theta} \widehat{\mathcal{R}}_{\rm s\mathchar`-adv}(\theta)$ as 
\begin{align}
\nabla_{\theta} \widehat{\mathcal{R}}_{\rm s\mathchar`-adv}(\theta) =\frac{1}{N} \sum_{s = 1}^S n_s w^{\ast}_s \nabla_{\theta}\widehat{\mathcal{R}}_s(\theta), 
\label{eq:gradient}
\end{align}
where $\vector{w}^{\ast} = (w^{\ast}_1, \ldots, w^{\ast}_S)$ is the solution of inner maximization of AERM in Eq.~\eqref{eq:general_opt_cluster}. 

In the following, we show that $\vector{w}^{\ast}$ can be obtained very efficiently for two well-known instances of $f$-divergences. 
\vspace{-0.5cm}
\paragraph{Kullback-Leibler (KL) divergence: }
For the KL divergence, $f(x) = x \log x$, 
we have
\begin{align}
w^{\ast}_s = \frac{N}{Z(\gamma)} \cdot {\rm exp} \left( \frac{\widehat{\mathcal{R}}_s(\theta)}{\gamma} \right),\ 1 \leq s \leq S \label{eq:kl1_opt},
\end{align}
where $\gamma$ is a scalar such that the first constraint of $\widehat{\mathcal{W}}_{f}$ in Eq.~\eqref{eq:uncertain_set_empirical_cluster} holds with equality, and $Z(\gamma) \equiv  \sum_{s = 1}^S n_s  {\rm exp}  \left( \widehat{\mathcal{R}}_s(\theta)/ \gamma \right)$ is a normalizing constant in order to satisfy the second constraint of $\widehat{\mathcal{W}}_{f}$. 
To compute $\gamma$, we can simply perform a binary search. 
\vspace{-0.2cm}
\paragraph{Pearson (PE) divergence: }
For the PE divergence, $f(x) = (x-1)^2$.
For small $\delta$, $\vector{w} \geq 0$ of $\widehat{\mathcal{W}}_{f}$ is often satisfied in practice.
We drop the inequality for simplicity; then, the solution of the inner maximization of Eq.~\eqref{eq:general_opt_cluster} becomes analytic and efficient to obtain:
\begin{align}
\vector{w}^{\ast} = \sqrt{\frac{N \delta}{\sum_{s = 1}^S n_s v_s^2}} \vector{v} + \vector{1}_S, \label{eq:r_analytical}
\end{align}
where ${\vector 1}_S$ is the $S$-dimensional vector with all the elements equal to 1. $\vector{v}$ is the $S$-dimensional vector such that $v_s = \widehat{\mathcal{R}}_s(\theta) - \widehat{\mathcal{R}}(\theta), \ 1 \leq s \leq S.$
\vspace{-0.2cm}
\paragraph{Computational complexity: } 
The time complexity for obtaining $\vector{w}^{\ast}$ is: $\mathcal{O}(mS)$ for the KL divergence and $\mathcal{O}(S)$ for the PE divergence, where $m$ is the number of the binary search iterations to compute $\gamma$ in Eq.~\eqref{eq:kl1_opt}. 
Calculating the adversarial weights therefore adds negligible computational overheads to computing $\nabla \ell_i(\theta)$ and $\ell_i(\theta)$ for $1 \leq i \leq N$, which for example requires $\mathcal{O}(Nb)$-time for a $b$-dimensional linear-in-parameter model.

\vspace{-0.2cm}
\section{Experiments} \label{sec:realdata}
\vspace{-0.2cm}

\setlength\intextsep{0pt}
\setlength\textfloatsep{0pt}
\begin{table*}[t]
\footnotesize
\begin{center} 
\caption{Experimental comparisons of the three methods w.r.t.~the estimated ordinary risk and the estimated structural adversarial risk \emph{using the 0-1 loss (\%)}. The lower these values are, the better the performance of the method is. The KL divergence is used and distribution shift is assumed to be (a) class prior change and (b) sub-category prior change. Mean and standard deviation over 50 random train-test splits were reported. The best method and comparable ones based on the t-test at the significance level 1\% are highlighted in boldface.} 
\label{table:kl}
\vspace{-0.1cm}
\subfloat[Class prior change.]{
\begin{tabular}{|l||c|c|c|c|c|c|c|}\hline
\raisebox{-1.3ex}{Dataset} & \multicolumn{3}{|c|}{\shortstack[c]{Estimated ordinary risk}} &\multicolumn{3}{|c|}{\shortstack[c]{Estimated structural adversarial risk}}  \\ \cline{2-7}
             & \raisebox{0ex}{ERM} & \raisebox{0ex}{AERM} &  \shortstack[c]{Structural AERM}
             & \raisebox{0ex}{ERM} & \raisebox{0ex}{AERM} &  \shortstack[c]{Structural AERM} \\ \hline
blood & {\bf 22.4 (0.7)} & 26.7 (5.0) & 33.4 (2.0) & 62.3 (2.4) & 53.5 (10.6) & {\bf 36.7 (1.9)}\\ \hline
adult & {\bf 15.3 (0.2)} & 15.4 (0.2) & 18.7 (0.2) & 30.4 (0.4) & 30.4 (0.5) & {\bf 19.1 (0.3)}\\ \hline
fourclass & {\bf 24.0 (1.4)} & 25.7 (2.6) & 27.2 (1.4) & 36.9 (2.2) & 38.0 (6.2) & {\bf 29.5 (2.0)}\\ \hline
phishing & {\bf 6.1 (0.2)} & 6.4 (0.2) & {\bf 6.0 (0.2)} & 7.6 (0.4) & 8.2 (0.4) & {\bf 6.5 (0.3)}\\ \hline
20news & {\bf 28.9 (0.4)} & 30.6 (0.4) & 34.7 (0.4) & 44.1 (0.5) & 45.0 (0.5) & {\bf 40.0 (0.6)}\\ \hline
satimage & {\bf 25.2 (0.3)} & 30.8 (0.3) & 32.2 (0.4) & {\bf 39.5 (0.6)} & 47.9 (0.5) & {\bf 39.6 (0.6)}\\ \hline
letter & {\bf 14.3 (0.5)} & 15.2 (0.6) & 19.3 (0.8) & 36.6 (1.5) & 34.7 (1.7) & {\bf 22.5 (1.0)}\\ \hline
mnist & {\bf 10.0 (0.1)} & 11.5 (0.1) & 12.7 (0.1) & 14.4 (0.2) & 15.9 (0.2) & {\bf 13.8 (0.2)}\\ \hline
\end{tabular}}
\\
\vspace{-0.2cm}
\subfloat[Sub-category prior change.]{
\begin{tabular}{|l||c|c|c|c|c|c|c|}\hline
\raisebox{-1.3ex}{Dataset} & \multicolumn{3}{|c|}{\shortstack[c]{Estimated ordinary risk}} &\multicolumn{3}{|c|}{\shortstack[c]{Estimated structural adversarial risk}}  \\ \cline{2-7}
             & \raisebox{0ex}{ERM} & \raisebox{0ex}{AERM} &  \shortstack[c]{Structural AERM}
             & \raisebox{0ex}{ERM} & \raisebox{0ex}{AERM} &  \shortstack[c]{Structural AERM} \\ \hline
20news & {\bf 19.0 (0.3)} & 20.6 (0.4) & 23.3 (0.5) & 35.6 (0.6) & 37.8 (0.9) & {\bf 31.1 (0.5)}\\ \hline
satimage & {\bf 36.4 (0.3)} & 44.2 (2.5) & 40.7 (0.9) & 62.1 (0.6) & 66.2 (4.2) & {\bf 50.5 (0.6)}\\ \hline
letter & {\bf 17.5 (0.4)} & 28.0 (5.4) & 34.2 (2.0) & 52.1 (0.6) & 60.1 (5.8) & {\bf 43.0 (1.7)}\\ \hline
mnist & {\bf 13.3 (0.1)} & 13.7 (0.1) & 17.3 (0.2) & 22.6 (0.2) & 23.2 (0.2) & {\bf 19.9 (0.2)}\\ \hline
\end{tabular}}
\end{center}
\vspace{-0.4cm}
\end{table*}

In this section, we experimentally analyze our DRSL (structural AERM) in classification by comparing it with ordinary ERM and DRSL with $f$-divergences (AERM).
We empirically demonstrate (i) the undesirability of AERM in classification and (ii) the robustness of structural AERM against specified distribution shift.
\vspace{-0.3cm}
\paragraph{Datasets:}
We obtained six classification datasets from the UCI repository \citep{uci}, two of which are for multi-class classification.
We also obtained MNIST \citep{lecun1998gradient} and 20newsgroups \citep{Lang95}. 
Refer to Appendix \ref{app:dataset} for the details of the datasets. 

\vspace{-0.3cm}
\paragraph{Evaluation metrics:}
We evaluated the three methods (ordinary ERM, AERM and structural AERM) with three kinds of metrics: the ordinary risk, the adversarial risk, and the structural adversarial risk, where the \emph{0-1 loss is used} for all the metrics.\footnote{To gain more insight on the methods, we also reported all the metrics in terms of the surrogate loss in Appendix \ref{app:exp_surrogate}.}
We did not explicitly report the adversarial risk in our experiments because of Theorem \ref{lem:monotonic}. 

Both the risk and structural adversarial risk are estimated using held-out test data. In particular, the structural adversarial risk can be estimated similarly to Eqs.~\eqref{eq:general_opt_cluster} and \eqref{eq:uncertain_set_empirical_cluster}, i.e., calculating the mis-classification rate on the held-out test data and \emph{structurally and adversarially reweight} them. 
See discussion of Eq.~\eqref{eq:decompose} for why the structural adversarial risk is a meaningful evaluation metric to measure distributional robustness of classifiers.

\vspace{-0.3cm}
\paragraph{Experimental protocols:}
For our DRSL, we consider learning classifiers robust against (a) the class prior change and (b) the sub-category prior change. This corresponds to grouping training data by (a) class labels and (b) sub-category labels, respectively. 
In the benchmark datasets, the sub-category labels are not available. 
Hence, we manually created such labels as follows.
First, we converted the original multi-class classification problems into classification problems with fewer classes by integrating some classes together.
Then, the original class labels are regarded as the subcategories.
In this way, we converted the satimage, letter and MNIST datasets into binary classification problems, and 20newsgroups into a 7-class classification. 
Appendix \ref{app:subcategory} details how we grouped the class labels.

For all the methods, we used linear models with softmax output for the prediction function $g_{\theta}(x)$.  The cross-entropy loss with $\ell_2$ regularization was adopted. The regularization hyper-parameter $\lambda$ was selected from $\{1.0, 0.1, 0.01, 0.001, 0.0001\}$ via 5-fold cross validation. 

We used the two $f$-divergences (the KL and PE divergences) and set $\delta = 0.5$ for AERM and structural AERM. The same $\delta$ and $f$-divergence were used for estimating the structural adversarial risk.
At the end of this section, we discuss how we can choose $\delta$ in practice.
\vspace{-0.4cm}
\paragraph{Results:}
In Table \ref{table:kl}, we report experimental results on the classification tasks when the KL divergence is used. 
Refer to Appendix \ref{app:exp_pe} for the results when the PE divergence is used, which showed similar tendency.

We see from the left half of Table \ref{table:kl} that ordinary ERM achieved lower estimated risks as expected. 
On the other hand, we see from the entire Table \ref{table:kl} that AERM, which does not incorporate any structural assumptions on distribution shift, performed poorly in terms of both of two evaluation metrics; hence, it also performed poorly in terms of the adversarial risk (see Theorem \ref{lem:monotonic}).  
This may be because AERM was excessively sensitive to outliers as implied by Theorem \ref{thm:steeper}.
We see from the right half of Table \ref{table:kl} that structural AERM achieved significantly lower estimated structural adversarial risks.
Although this was expected, our experiments confirmed that structural AERM indeed obtained classifiers robust against the \emph{structural} distribution shift.\footnote{When we used the surrogate loss to evaluate the methods (which is not the case in ordinary classification), we confirmed that the methods indeed achieved the best performance in terms of the metrics they optimized for, i.e., ERM, AERM, and structural AERM performed the best in terms of the ordinary risk, adversarial risk and structural adversarial risk, respectively. See Appendix \ref{app:exp_surrogate} for the actual experimental results.}
\vspace{-0.5cm}
\paragraph{Discussion: }
Here we provide an insight for users to determine $\delta$ in our DRSL (structural ARM and AERM). 
We see from Eq.~\eqref{eq:decompose} that the structural adversarial risk can be decomposed into the sum of the ordinary risk and the sensitivity term, where $\delta$ acts as a trade-off hyper-parameter between the two terms.
In practice, users of our DRSL may want to have good balance between the two terms, i.e., the learned classifier should achieve high accuracy on the training distribution while being robust to specified distribution shift.
Since both terms in Eq.~\eqref{eq:decompose} can be estimated by cross validation, the users can adjust $\delta$ of AERM at training time to best trade-off the two terms for their purposes, e.g., increasing $\delta$ during training to decrease the sensitivity term at the expense of a slight increase of the risk term.
\vspace{-0.2cm}
\section{Conclusion} \label{sec:conclusion}
\vspace{-0.2cm}
In this paper, we theoretically analyzed DRSL with $f$-divergences applied to classification. 
We showed that the DRSL ends up giving a classifier optimal for the \emph{training distribution}, which is too pessimistic in terms of the original motivation of distributionally robust classification.
To rectify this, we presented simple DRSL that gives a robust classifier based on structural assumptions on distribution shift. 
We derived efficient optimization algorithms for our DRSL and empirically demonstrated its effectiveness.

\section*{Acknowledgement}
We thank anonymous reviewers for their constructive feedback.
WH was supported by JSPS KAKENHI 18J22289.
MS was supported by CREST JPMJCR1403.

\bibliographystyle{icml2018}
\bibliography{reference}

\onecolumn
\appendix

\section{Proof of Theorem \ref{lem:monotonic}} \label{app:proof_monotonic}
Given $\theta$ (the parameter of the prediction function), we obtain the adversarial risk by the following optimization:
\begin{align}
&\max_{r \in \mathcal{U}_f} \mathbb{E}_{p(x,y)} [r(x, y) \ell(g_{\theta}(x), y)]  \label{eq:reform} \\
&\text{where} \ \ \  \mathcal{U}_f \equiv \{ r(x,y)\  |\ \mathbb{E}_{p(x, y)} \left[ f \left( r(x,y) \right)\right] \leq \delta,\  \mathbb{E}_{p(x,y)}[r(x,y)] = 1, \  r(x, y) \geq 0 \ \left( \forall (x, y) \in \mathcal{X} \times \mathcal{Y} \right)\}. \label{eq:feasible}
\end{align}

Here, we are considering that $\ell(\cdot, \cdot)$ is the 0-1 loss. Let $\Omega^{(0)}_{\theta} \equiv \{(x,y) \ | \ \ell(g_{\theta}(x),y) = 0 \} \subseteq \mathcal{X} \times \mathcal{Y}$.
Then, we have $\Omega_{\theta}^{(1)} \equiv \{(x,y) \ | \ \ell(g_{\theta}(x),y) = 1 \} = \mathcal{X} \times \mathcal{Y} \setminus \Omega_{\theta}$.
Let $r^{\ast}(\cdot,\cdot)$ be the optimal solution of Eq.~\eqref{eq:reform}.

Note that Eq.~\eqref{eq:reform} is a convex optimization problem. This is because the objective is linear in $r(\cdot, \cdot)$ and the uncertainty set is convex, which follows from the fact that $f(\cdot)$ is convex. Therefore, any local maximum of Eq.~\eqref{eq:reform} is the global maximum.
Nonetheless, there can be multiple solutions that attain the same global maxima.
Among those solutions, we now show that there exists $r^{\ast}(\cdot, \cdot)$ such that it takes the same values within $\Omega_{\theta}^{(0)}$ and $\Omega_{\theta}^{(1)}$, respectively, i.e., 
\begin{align}
r^{\ast}(x,y) = r^{\ast}_0 \ \ \text{for}\ \  \forall (x,y) \in \Omega^{(0)}_{\theta}, \qquad r^{\ast}(x,y) = r^{\ast}_1 \ \ \text{for}\ \  \forall (x,y) \in \Omega_{\theta}^{(1)}, \label{eq:uniform}
\end{align}
where $r^{\ast}_0$ and $r^{\ast}_1$ are some constant values.
This is because for any given optimal solution $r^{\prime \ast}(\cdot, \cdot)$ of Eq.~\eqref{eq:reform}, we can always obtain optimal solution $r^{\ast}(\cdot, \cdot)$ that satisfies Eq.~\eqref{eq:uniform} by the following transformation:
\begin{align}
r^{\ast}_0 = \mathbb{E}_{p(x,y)}[r^{\prime \ast}(x,y) \cdot \vector{1}\{(x,y)\in \Omega_{\theta}^{(0)}\}] / \mathbb{E}_{p(x,y)}[\vector{1}\{(x,y)\in \Omega^{(0)}_{\theta}\}], \label{eq:transform0} \\
r^{\ast}_1 = \mathbb{E}_{p(x,y)}[r^{\prime \ast}(x,y) \cdot \vector{1}\{(x,y)\in \Omega_{\theta}^{(1)}\}] / \mathbb{E}_{p(x,y)}[\vector{1}\{(x,y)\in \Omega_{\theta}^{(1)}\}], \label{eq:transform1}
\end{align}
where $\vector{1}\{\cdot\}$ is an indicator function.
Eqs.~\eqref{eq:transform0} and \eqref{eq:transform1} are simple average operations of $r^{\prime \ast}(\cdot,\cdot)$ on regions $\Omega_{\theta}^{(0)}$ and $\Omega_{\theta}^{(1)}$, respectively.
Utilizing the convexity of $f(\cdot)$, it is straightforward to see that $r^{\ast}(\cdot, \cdot)$ constructed in this way is still in the feasible region of Eq.~\eqref{eq:feasible}. It also attains exactly the same objective of Eq.~\eqref{eq:reform} as the original solution $r^{\ast \prime}(\cdot, \cdot)$ does. This concludes our proof that there exists an optimal solution of Eq.~\eqref{eq:reform} such that it takes the same values within $\Omega_{\theta}^{(0)}$ and $\Omega_{\theta}^{(1)}$, respectively.

Let $p_{\Omega^{(i)}_{\theta}} = \mathbb{E}_{p(x,y)}[\vector{1}\{(x,y)\in \Omega^{(i)}_{\theta}\}]$ for $i = 0, 1$, which is the proportion of data that is correctly and incorrectly classified, respectively.
We note that $p_{\Omega^{(0)}_{\theta}} + p_{\Omega^{(1)}_{\theta}} = 1.$ 
Also, we see that $p_{\Omega^{(1)}_{\theta}}$ is by definition the misclassification rate; thus, it is equal to the ordinary risk, i.e., $\mathbb{E}_{p(x, y)} [\ell(g_{\theta}(x), y)].$
By using Eq.~\eqref{eq:uniform}, we can simplify Eq.~\eqref{eq:reform} as
\begin{align}
&\sup_{(r_0, r_1) \in \mathcal{U}^{\prime}_f} p_{\Omega^{(1)}_{\theta}} r_1 , \label{eq:obj_binary} \\
& \text{where} \ \ \ \mathcal{U}_f^{\prime} \equiv \{ (r_0, r_1) \ | 
\ p_{\Omega^{(0)}_{\theta}} f(r_0) + p_{\Omega^{(1)}_{\theta}} f(r_1) \leq \delta,\  p_{\Omega^{(0)}_{\theta}} r_0 + p_{\Omega^{(1)}_{\theta}}  r_1 = 1,\  r_0 \geq 0,\  r_1 \geq 0\}. \label{eq:uncertain_binary}
\end{align} 

In the following, we show that Eq.~\eqref{eq:obj_binary} has monotonic relationship with $p_{\Omega^{(1)}_{\theta}}.$
With a fixed value of $p_{\Omega^{(1)}_{\theta}}$, we can obtain the optimal $(r_0, r_1)$ by solving Eq.~\eqref{eq:obj_binary}.
Let $(r_0^{\ast}(p_1), r_1^{\ast}(p_1))$ be the solution of Eq.~\eqref{eq:obj_binary} when $p_{\Omega^{(1)}_{\theta}}$ is fixed to $p_1$ with $0 \leq p_1 \leq 1$.

First, we note that the first inequality constraint in Eq.~\eqref{eq:uncertain_binary} is a convex set and includes $(1,1)$ in its relative interior because $p_{\Omega^{(0)}_{\theta}} f(1) + p_{\Omega^{(1)}_{\theta}} f(1) =  0 < \delta$ and $p_{\Omega^{(0)}_{\theta}} \cdot 1 + p_{\Omega^{(1)}_{\theta}}  \cdot 1 = 1$.
Note that in a two dimensional space, the number of intersections between a line and a boundary of a convex set is at most two.
For $\delta > 0$, there are always exactly two different points that satisfies both $p_{\Omega^{(0)}_{\theta}} f(r_0) + p_{\Omega^{(1)}_{\theta}} f(r_1) = \delta$ and $p_{\Omega^{(0)}_{\theta}} r_0 + p_{\Omega^{(1)}_{\theta}}  r_1 = 1.$
We further see that the optimal solution of $r_1$ is always greater than 1 because the objective in Eq.~\eqref{eq:obj_binary} is an increasing function of  $r_1$.
Taking these facts into account, we can see that the optimal solution, $(r_0^{\ast}(p_1), r_0^{\ast}(p_1))$, satisfies either of the following two cases depending on whether the inequality constraint $r_0 \geq 0$ in Eq.~\eqref{eq:uncertain_binary} is active or not. 
\begin{align}
&\text{{\bf Case 1:}}\qquad p_0 \cdot f(r_0^{\ast}(p_1)) + p_1 \cdot f(r_1^{\ast}(p_1)) = \delta,\ \ p_0 \cdot r_0^{\ast}(p_1) + p_1 \cdot r_1^{\ast}(p_1) = 1,\ \  0 < r_0^{\ast}(p_1) < 1 < r_1^{\ast}(p_1). \label{eq:case1} \\
&\text{{\bf Case 2:}}\qquad r_0^{\ast}(p_1) = 0,  \ \ r_1^{\ast}(p_1) = \frac{1}{p_1},\label{eq:case2}
\end{align}
where $p_0 = 1 - p_1.$ 
We now show that there is the monotonic relation between $p_{\Omega^{(1)}_{\theta}}$ and Eq.~\eqref{eq:obj_binary} for both the cases.
To this end, pick any $p_1^{\prime}$ such that $p_1 < p_1^{\prime} \leq 1$, and let $(r_0^{\ast}(p_1^{\prime}), r_1^{\ast}(p_1^{\prime}))$ be the solution of Eq.~\eqref{eq:obj_binary} when $p_{\Omega^{(1)}_{\theta}}$ is fixed to $p_1^{\prime}.$

Regarding Case 2 in Eq.~\eqref{eq:case2}, the adversarial risk in Eq.~\eqref{eq:obj_binary} is $p_1 \cdot \frac{1}{p_1} = 1$.
On the other hand, it is easy to see that the active equality constraint stays $r_0 \geq 0$ in Eq.~\eqref{eq:uncertain_binary} for $p_{\Omega^{(1)}_{\theta}}$ larger $p_1$. Hence, we can show that $r_0^{\ast}(p_1^{\prime}) = 0, r_1^{\ast}(p_1^{\prime}) = \frac{1}{p_1^{\prime}}$. Therefore, the adversarial risk in Eq.~\eqref{eq:obj_binary} stays $p_1^{\prime} \cdot \frac{1}{p_1^{\prime}} = 1$. This concludes our proof for Eq.~\eqref{eq:monotonic_erm2} in Theorem \ref{lem:monotonic}.

Regarding Case 1 in Eq.~\eqref{eq:case1}, we note that both the ordinary risk $p_1$ and the adversarial risk $p_1 \cdot r_1^{\ast}(p_1)$ are strictly less than 1.
Our goal is to show Eq.~\eqref{eq:monotonic_erm} in Theorem \ref{lem:monotonic}, which is equivalent to showing
\begin{align}
p_1 \cdot r_1^{\ast}(p_1) < p_1^{\prime} \cdot r_1^{\ast}(p_1^{\prime}). \label{eq:target_monotonic}
\end{align} 
To do so, we further consider the following two sub-cases of Case 1 in Eq.~\eqref{eq:case1}:
\begin{align}
&\text{{\bf Case 1-a}:}\qquad p_1^{\prime} < p_1 \cdot r_1^{\ast}(p_1) \label{eq:case1-a}\\
&\text{{\bf Case 1-b}:}\qquad p_1^{\prime} \geq p_1 \cdot r_1^{\ast}(p_1).
\end{align}
In Case 1-b, we can straightforwardly show Eq.~\eqref{eq:target_monotonic} as follows.
\begin{align}
p_1 \cdot r_1^{\ast}(p_1) \leq p_1^{\prime} < p_1^{\prime} \cdot r_1^{\ast}(p_1^{\prime}),
\end{align}
where the last inequality follows from $1 < r_1^{\ast}(p_1^{\prime}).$

Now, assume Cases 1 and 1-a in Eqs.~\eqref{eq:case1} and \eqref{eq:case1-a}. Our goal is to show that $r_1^{\ast}(p_1^{\prime})$ satisfies $r_1^{\ast}(p_1^{\prime}) > \frac{p_1}{p_1^{\prime}} r_1^{\ast}(p_1)$ because then Eq.~\eqref{eq:target_monotonic} holds.
To this end, we show that 
\begin{align}
r_1^{\prime} = \frac{p_1}{p_1^{\prime}} \cdot r_1^{\ast}(p_1) \label{eq:yeye}
\end{align}
is contained in the relative interior (excluding the boundary) of Eq.~\eqref{eq:uncertain_binary} with $p_{\Omega^{(1)}_{\theta}}$ fixed to $p_1^{\prime}.$ Then, because our objective in Eq.~\eqref{eq:obj_binary} is linear in $r_1$, $r_1^{\prime} < r_1^{\ast}(p_1^{\prime})$ holds in our setting. Then, we arrive at $r_1^{\ast}(p_1^{\prime}) > \frac{p_1}{p_1^{\prime}} \cdot r_1^{\ast}(p_1)$.
Formally, our goal is to show $r_1^{\prime}$ in Eq.~\eqref{eq:yeye} satisfies
\begin{align}
p_0^{\prime} \cdot f(r_0^{\prime}) + p_1^{\prime} \cdot f(r_1^{\prime}) < \delta,\ \ \  p_0^{\prime} r_0^{\prime} + p_1^{\prime} r_1^{\prime} = 1,\ \ \  r_0^{\prime} > 0,\  \ \ r_1^{\prime} > 0, \label{eq:uncertainty_included}
\end{align}
where $p_0^{\prime} = 1 - p_1^{\prime}.$ 
By Eqs.~\eqref{eq:case1}, \eqref{eq:yeye} and the second equality of Eq.~\eqref{eq:uncertainty_included}, we have 
\begin{align}
r_0^{\prime} = \frac{1 - p_1^{\prime}r_1^{\prime}}{p_0^{\prime}} = \frac{1 - p_1 \cdot r_1^{\ast}(p_1)}{p_0^{\prime}} = \frac{p_0}{p_0^{\prime}} \cdot r_0^{\ast}(p_1). \label{eq:yiyi}
\end{align}

The latter two inequalities of Eq.~\eqref{eq:uncertainty_included}, i.e., $r_0^{\prime} > 0$ and $r_1^{\prime} > 0$ follow straightforwardly from the assumptions. Combining the assumption in Eq.~\eqref{eq:case1-a} and the last inequality in Eq.~\eqref{eq:case1}, we have the following inequality.
\begin{align}
0 < r_0^{\ast}(p_1) < r_0^{\prime} < 1 < r_1^{\prime} < r_1^{\ast}(p_1).
\end{align}
Thus, we can write $r_0^{\prime}$ (resp. $r_1^{\prime}$) as a linear interpolation of $r_0^{\ast}(p_1)$ and 1 (resp. 1 and $r_1^{\ast}(p_1)$) as follows.
\begin{align}
r_0^{\prime} = \alpha \cdot r_0^{\ast}(p_1) + (1-\alpha)\cdot 1, \qquad r_1^{\prime} = \beta \cdot r_1^{\ast}(p_1) + (1 - \beta) \cdot 1, 
\end{align}  
where $0 < \alpha, \beta < 1$. 
Substituting Eqs.~\eqref{eq:yeye} and \eqref{eq:yiyi}, we have
\begin{align}
\alpha &= \frac{1}{1 - r_0^{\ast}(p_1)} \cdot \frac{p_0^{\prime} - p_0 \cdot r_0^{\ast}(p_1)}{p_0^{\prime}}, \label{eq:hihi}\\
\beta &= \frac{1}{r_1^{\ast}(p_1) - 1} \cdot \frac{p_1 \cdot r_1^{\ast}(p_1) - p_1^{\prime}}{p_1^{\prime}} = \frac{1}{r_1^{\ast}(p_1) - 1} \cdot \frac{p_0^{\prime} - p_0 \cdot r_0^{\ast}(p_1)}{p_1^{\prime}}.  \label{eq:hehe}
\end{align}

Then, we have
\begin{align*}
p_0^{\prime} f(r_0^{\prime}) + p_1^{\prime} f(r_1^{\prime}) &= p_0^{\prime} f(\alpha \cdot r_0^{\ast}(p_1) + (1-\alpha)\cdot 1) + p_1^{\prime} f(\beta \cdot r_1^{\ast}(p_1) + (1 - \beta) \cdot 1) \\
& \leq p_0^{\prime} \cdot \{ \alpha \cdot f(r_0^{\ast}(p_1)) + (1-\alpha) \cdot f(1) \} + p_1^{\prime} \cdot \{ \beta \cdot f(r_1^{\ast}(p_1)) + (1-\beta)\cdot f(1) \} \ \ \text{($\because$ convexity of $f(\cdot)$)}\\
& = p_0^{\prime} \alpha \cdot f(r_0^{\ast}(p_1)) + p_1^{\prime}\beta \cdot f(r_1^{\ast}(p_1)) \ \ \text{($\because$ $f(1) = 0$)} \\
& = (p_0^{\prime} - p_0 \cdot r_0^{\ast}(p_1)) \left( \frac{1}{1 - r_0^{\ast}(p_1)} \cdot f(r_0^{\ast}(p_1)) + \frac{1}{r_1^{\ast}(p_1) - 1} \cdot f(r_1^{\ast}(p_1))  \right) \ \ \left( \because \text{Eqs.~\eqref{eq:hihi} and \eqref{eq:hehe}} \right)\\
& = (p_1 \cdot r_1^{\ast}(p_1) - p_1^{\prime} )\left( \frac{1}{1 - r_0^{\ast}(p_1)} \cdot f(r_0^{\ast}(p_1)) + \frac{1}{r_1^{\ast}(p_1) - 1} \cdot f(r_1^{\ast}(p_1))  \right) \\
& < (p_1 \cdot r_1^{\ast}(p_1) - p_1) \left( \frac{1}{1 - r_0^{\ast}(p_1)}\cdot f(r_0^{\ast}(p_1)) + \frac{1}{r_1^{\ast}(p_1) - 1} \cdot f(r_1^{\ast}(p_1))  \right) \ \ \text{($\because$ $p_1^{\prime} > p_1.$)} \\
& =  p_0 \cdot f(r_0^{\ast}(p_1)) + p_1 \cdot f(r_1^{\ast}(p_1)) \\
& = \delta. \ \ \text{($\because$ the first equation of Eq.~\eqref{eq:case1}.)}
\end{align*}

This concludes our proof that Eq.~\eqref{eq:obj_binary} has monotonic relationship with $p_{\Omega^{(1)}_{\theta}}.$
Recall that $p_{\Omega^{(1)}_{\theta}}$ is by definition equal to the ordinary risk, $\mathcal{R}(\theta).$
Therefore, for any pair of parameters $\theta_1$ and $\theta_2$, if $\mathcal{R}_{\rm adv}(\theta_1) < 1$, we have 
\begin{align} 
\mathcal{R}(\theta_1) < \mathcal{R}(\theta_2) \Longrightarrow  \mathcal{R}_{\rm adv}(\theta_1) < \mathcal{R}_{\rm adv}(\theta_2). \label{eq:bnc}
\end{align}
To show that the opposite direction of Eq.~\eqref{eq:bnc} holds, we need to show that any pair of parameters $\theta_1$ and $\theta_2$, the following holds:
\begin{align}
\mathcal{R}(\theta_1) = \mathcal{R}(\theta_2) \Longrightarrow \mathcal{R}_{\rm adv}(\theta_1) = \mathcal{R}_{\rm adv}(\theta_2).
\end{align}
This is obvious from Eq.~\eqref{eq:obj_binary} because the adversarial risk depends on the parameter of the model \emph{only through the risk of the model.}
This concludes the proof of Theorem \ref{lem:monotonic}, in which the adversarial risk and ordinary risk are compared.
For the case of empirical approximations, the same argument can be used by replacing the expectations with empirical averages. \qed

\section{Proof of Theorem \ref{thm:calibrated}} \label{app:proof_calibrated}
We prove by contradiction.
Let $\Omega$ be the subset of $\mathbb{R}^{|\mathcal{Y}|}$.
We consider the multi-class classification and assume that the loss $\ell(\cdot, \cdot): \Omega \times \mathcal{Y} \rightarrow \mathbb{R}_{\geq 0}$ is classification-calibrated. Although we will mainly focus on a multi-class classification scenario, our proof easily extends to a binary classification scenario, which we will discuss at the end of the proof.

Let $g(\cdot): \mathcal{X} \to \mathbb{R}^{K}$ be prediction function, where $K$ is the number of classes. Assume the prediction function, $g$, can take any measurable functions.
Then, $g^{\ast}$ that minimizes the ordinary risk using the classification-calibrated loss, $\ell(\cdot, \cdot)$, i.e.,
\begin{align}
g^{\ast} &= \argmin_{g} \mathbb{E}_{p(x, y)} \left[ \ell(g(x), y) \right] \nonumber \\
& = \argmin_{g} \mathbb{E}_{p(x)} \left[ \sum_{y \in \mathcal{Y}} p(y|x) \ell(g(x), y) \right], \label{eq:ordinary_rm}
\end{align}
is the Bayes optimal classifier\footnote{The classifier that minimizes the mis-classification rate for the training density $p(x, y)$ (the 0-1 loss is considered), i.e., the classifier whose prediction on $x$ is equal to $\argmax_{y \in \mathcal{Y}} p(y | x)$.} \citep{bartlett2006convexity, tewari2007consistency}.

Our goal is to show that $g^{({\rm adv})}$ that minimizes the adversarial risk using classification-calibrated loss is also Bayes optimal w.r.t.~$p(x,y).$ 
More specifically, we consider
\begin{align}
g^{({\rm adv})} &= \argmin_{g} \sup_{q : {\rm D}_f[q || p]\leq \delta} \mathbb{E}_{q(x,y)} [\ell(g(x), y)] \label{eq:inner_again2} \\
& =  \argmin_{g} \sup_{ r(\cdot, \cdot)  \in \mathcal{U}_f } \mathbb{E}_{p(x, y)} \left[r(x, y)  \ell(g(x), y) \right]. \label{eq:inner_again}
\end{align}
Recall that $q(x, y)$ in Eq.~\eqref{eq:inner_again2} and $r(x, y)$ in Eq.~\eqref{eq:inner_again} are related by $r(x, y) \equiv q(x, y)/p(x, y)$.
In the following, with a slight abuse of notation, we denote $r(x) \equiv q(x)/p(x)$ and $r(y|x) \equiv q(y|x)/p(y|x)$. Obviously, we have $r(x, y) = r(x)r(y|x)$.

Let $r^{\ast}(\cdot, \cdot)$ be the solution of the inner maximization of Eq.~\eqref{eq:inner_again} with $g^{\rm (adv)}$, i.e.,
\begin{align}
r^{\ast}(\cdot, \cdot) = \argmax_{ r(\cdot,\cdot)  \in \mathcal{U}_f } \mathbb{E}_{p(x, y)} \left[r(x, y)  \ell(g^{\rm (adv)}(x), y) \right]. \label{eq:sugi}
\end{align}
Then, by Danskin's theorem \cite{danskin67}, Eq.~\eqref{eq:inner_again} can be rewritten as
\begin{align}
g^{({\rm adv})} &= \argmin_{g} \mathbb{E}_{p(x, y)} \left[r^{\ast}(x, y)  \ell(g(x), y) \right] \nonumber \\
& = \argmin_{g} \mathbb{E}_{p(x)} \left[r^{\ast}(x) \mathbb{E}_{p(y|x)}[ r^{\ast}(y|x) \ell(g(x), y)] \right] \nonumber \\
& = \argmin_{g} \mathbb{E}_{p(x)} \left[r^{\ast}(x)  \sum_{y \in \mathcal{Y}} p(y|x) r^{\ast}(y|x) \ell(g(x), y) \right]. \label{eq:inner}
\end{align}
Now, suppose that $g^{({\rm adv})}$ is not Bayes optimal almost surely over $q^{\ast}(x) \equiv r^{\ast}(x) p(x)$. 
Then, we have
\begin{align}
\int_{x \in \mathcal{S}} q^{\ast}(x) {\rm d}x > 0, \label{eq:assume}
\end{align}
where 
\begin{align}
\mathcal{S} \equiv  \left\{ x \ \middle| \ x \in \mathcal{X},\ p(x) > 0, \ q^{\ast}(x) > 0, \  \argmax_{y \in \mathcal{Y}} p(y|x) \neq \argmax_{y \in \mathcal{Y}} g^{({\rm adv})}_{y}(x) \right\}.
\end{align}

In the following, we denote $x \in \mathcal{S}$ by $x^{\dagger}$, i.e., whenever we denote $x^{\dagger}$, we implicitly assume $x^{\dagger} \in \mathcal{S}$. 
We immediately have $r^{\ast}(x^{\dagger}) = q^{\ast}(x^{\dagger})/p(x^{\dagger}) > 0$.
We let $y^{{\rm (max)}}(x^{\dagger}) \equiv \argmax_{y \in \mathcal{Y}} p(y|x^{\dagger})$ and $y^{{\rm (adv)}}(x^{\dagger}) \equiv \argmax_{y \in \mathcal{Y}} g^{{\rm (adv)}}_{y}(x^{\dagger}).$
Since $\ell(\cdot, \cdot)$ is classification-calibrated, from Eq.~\eqref{eq:inner} and the definition of the classification-calibrated loss \citep{bartlett2006convexity, tewari2007consistency}, we have
\begin{align}
y^{{\rm (adv)}}(x^{\dagger}) = \argmax_{y \in \mathcal{Y}} p(y|x^{\dagger}) r^{\ast}(y|x^{\dagger}).
\end{align} 
Because we have $x^{\dagger} \in \mathcal{S}$, $y^{{\rm (adv)}}(x^{\dagger}) \neq y^{{\rm (max)}}(x^{\dagger})$ holds. Thus, we have
\begin{align}
p(y^{{\rm (adv)}}(x^{\dagger})|x^{\dagger}) r^{\ast}(y^{{\rm (adv)}}(x^{\dagger}) |x^{\dagger}) > p(y^{{\rm (max)}}(x^{\dagger})|x^{\dagger}) r^{\ast}(y^{{\rm (max)}}(x^{\dagger})|x^{\dagger}). \label{eq:wei}
\end{align}
Combining this with $p(y^{{\rm (max)}}(x^{\dagger})|x^{\dagger}) > p(y^{{\rm (adv)}}(x^{\dagger})|x^{\dagger})$, we have 
\begin{align}
r^{\ast}(y^{{\rm (adv)}}(x^{\dagger})|x^{\dagger}) > r^{\ast}(y^{{\rm (max)}}(x^{\dagger})|x^{\dagger}) > 0. \label{eq:bigger}
\end{align}
We construct a new ratio function, $r_{\rm new}(\cdot, \cdot)$, by the following operations. We first set $r_{\rm new}(x,y) \leftarrow r^{\ast}(x,y)$ for all $(x,y) \in \mathcal{X} \times \mathcal{Y}$. Then, for $\forall x \in \mathcal{S}$, we let
\begin{align}
r_{\rm new}(y^{{\rm (max)}}(x) | x) & \leftarrow r^{\ast}(y^{{\rm (max)}}(x) | x) + \epsilon \cdot p(y^{{\rm (adv)}}(x) | x), \label{eq:swapping1} \\
r_{\rm new}(y^{{\rm (adv)}}(x) | x) &  \leftarrow r^{\ast}(y^{{\rm (adv)}}(x) | x) - \epsilon \cdot p(y^{{\rm (max)}}(x) | x), \label{eq:swapping2}
\end{align}
where $\epsilon > 0$ is a sufficiently small number.
We show that such $r_{\rm new}(\cdot, \cdot)$ is still in $\mathcal{U}_f$. As shown in Eqs.~\eqref{eq:swapping1} and \eqref{eq:swapping2}, the value of $r_{\rm new}(\cdot, \cdot)$ changed from $r^{\ast}(\cdot, \cdot)$ only in $\mathcal{S}$. Therefore, given that $r^{\ast}(\cdot, \cdot) \in \mathcal{U}_f$, in order to show $r_{\rm new}(\cdot, \cdot) \in \mathcal{U}_f$, it is sufficient to show the following three equality/inequalities:
\begin{align}
& \mathbb{E}_{p(x, y)} \left[ f \left( r_{\rm new}(x,y) \right)\right] \leq \delta,\ \mathbb{E}_{p(x,y)}[r_{\rm new}(x,y)] = 1, \ r_{\rm new}(x, y) \geq 0 \ (\forall (x, y) \in \mathcal{X} \times \mathcal{Y}). 
\end{align}
Because we know $r^{\ast}(\cdot, \cdot) \in \mathcal{U}_f$, it is sufficient to show
\begin{align}
& \mathbb{E}_{p(x, y)} \left[ f \left( r_{\rm new}(x,y) \right)\right] \leq \mathbb{E}_{p(x, y)} \left[ f \left( r^{\ast}(x,y) \right)\right], \nonumber \\ 
&\mathbb{E}_{p(x,y)}[r_{\rm new}(x,y)] = \mathbb{E}_{p(x,y)}[r^{\ast}(x,y)], \nonumber \\
& r_{\rm new}(x, y) \geq 0, \forall (x, y) \in \mathcal{X} \times \mathcal{Y}. \label{eq:camera}
\end{align}
Furthermore, $r_{\rm new}(\cdot, \cdot)$ only differs from $r^{\ast}(\cdot, \cdot)$ in $(x, y^{\rm (max)}) \in \mathcal{S} \times \mathcal{Y}$ and $(x, y^{\rm (adv)}) \in \mathcal{S} \times \mathcal{Y}$. Therefore, to show Eq.~\eqref{eq:camera}, all we need to show is
\begin{align}
&\int_{x \in \mathcal{S}} \left\{p(x, y^{{\rm (adv)}}(x)) \cdot f(r_{\rm new}(x, y^{{\rm (adv)}}(x))) + p(x^{\dagger}, y^{{\rm (max)}}(x)) \cdot f(r_{\rm new}(x, y^{{\rm (max)}}(x))) \right\} {\rm d}x \nonumber \\ 
&\qquad \leq  \int_{x \in \mathcal{S}} \left\{ p(x, y^{{\rm (adv)}}(x)) \cdot f(r^{\ast}(x, y^{{\rm (adv)}}(x))) + p(x, y^{{\rm (max)}}(x)) \cdot f(r^{\ast}(x, y^{{\rm (max)}}(x))) \right\} {\rm d}x, \label{eq:first} \\
&\int_{x \in \mathcal{S}} \left\{ p(x, y^{{\rm (adv)}}(x)) \cdot r_{\rm new}(x, y^{{\rm (adv)}}(x)) + p(x, y^{{\rm (max)}}(x)) \cdot r_{\rm new}(x, y^{{\rm (max)}}(x)) \right\} {\rm d}x   \nonumber \\
&\qquad = \int_{x \in \mathcal{S}} \left\{ p(x, y^{{\rm (adv)}}(x)) \cdot r^{\ast}(x, y^{{\rm (adv)}}(x)) + p(x, y^{{\rm (max)}}(x)) \cdot r^{\ast}(x, y^{{\rm (max)}}(x)) \right\} {\rm d}x,  \label{eq:second} \\
& r_{\rm new}(x,y^{{\rm (adv)}}(x)) \geq 0, \ \  r_{\rm new}(x,y^{{\rm (max)}}(x)) \geq 0, \ \forall x \in \mathcal{S}. \label{eq:third}
\end{align}
First, since $r^{\ast}(y^{{\rm (adv)}}(x^{\dagger}) |x^{\dagger}) > 0$ holds from Eq.~\eqref{eq:wei}, $r_{\rm new}(y^{{\rm (adv)}}(x^{\dagger}) | x^{\dagger}) =  r^{\ast}(y^{{\rm (adv)}}(x^{\dagger}) | x^{\dagger}) - \epsilon \cdot p(y^{{\rm (max)}}(x^{\dagger}) | x^{\dagger}) \geq 0$ holds for sufficiently small $\epsilon$.
Thus, $r_{\rm new}(x^{\dagger}, y^{{\rm (adv)}}(x^{\dagger}) ) = r_{\rm new}(y^{{\rm (adv)}}(x^{\dagger}) | x^{\dagger})r(x^{\dagger}) \geq 0$. Hence, Eq.~\eqref{eq:third} holds for sufficiently small $\epsilon$.
Also, Eq.~\eqref{eq:second} follows because
\begin{align}
&\text{Integrand of L.H.S.~of Eq.~\eqref{eq:second}} \nonumber \\
&\qquad = p(x^{\dagger}, y^{{\rm (adv)}}(x^{\dagger})) \cdot \left[ r_{\rm new}(x^{\dagger}) \cdot \left\{ r^{\ast}(y^{{\rm (adv)}}(x^{\dagger}) | x^{\dagger}) - \epsilon \cdot p(y^{{\rm (max)}}(x^{\dagger}) | x^{\dagger}) \right\}   \right] \nonumber \\
& \qquad \qquad \qquad + p(x^{\dagger}, y^{{\rm (max)}}(x^{\dagger})) \cdot \left[ r_{\rm new}(x^{\dagger}) \cdot \left\{ r^{\ast}(y^{{\rm (max)}}(x^{\dagger}) | x^{\dagger}) + \epsilon \cdot p(y^{{\rm (adv)}}(x^{\dagger}) | x^{\dagger}) \right\}   \right] \\
&\qquad = p(x^{\dagger}, y^{{\rm (adv)}}(x^{\dagger})) \cdot \left[ r^{\ast}(x^{\dagger}) \cdot \left\{ r^{\ast}(y^{{\rm (adv)}}(x^{\dagger}) | x^{\dagger}) - \epsilon \cdot p(y^{{\rm (max)}}(x^{\dagger}) | x^{\dagger}) \right\}   \right] \nonumber \\
& \qquad \qquad \qquad + p(x^{\dagger}, y^{{\rm (max)}}(x^{\dagger})) \cdot \left[ r^{\ast}(x^{\dagger}) \cdot \left\{ r^{\ast}(y^{{\rm (max)}}(x^{\dagger}) | x^{\dagger}) + \epsilon \cdot p(y^{{\rm (adv)}}(x^{\dagger}) | x^{\dagger}) \right\}   \right] \\
&\qquad  = p(x^{\dagger}, y^{{\rm (adv)}}(x^{\dagger})) \cdot r^{\ast}(x^{\dagger}, y^{{\rm (adv)}}(x^{\dagger})) + p(x^{\dagger}, y^{{\rm (max)}}(x^{\dagger})) \cdot r^{\ast}(x^{\dagger}, y^{{\rm (max)}}(x^{\dagger}))  \nonumber \\
& \qquad \qquad \qquad + \epsilon \cdot r^{\ast}(x^{\dagger}) p(x^{\dagger}) \underbrace{\left\{ p(y^{{\rm (adv)}}(x^{\dagger}) | x^{\dagger}) p(y^{{\rm (max)}}(x^{\dagger}) | x^{\dagger}) - p(y^{{\rm (max)}}(x^{\dagger}) | x^{\dagger}) p(y^{{\rm (adv)}}(x^{\dagger}) | x^{\dagger}) \right\}}_{= 0} \\
&\qquad  = \text{Integrand of R.H.S.~of Eq.~\eqref{eq:second}}
\end{align}
Finally, we show Eq.~\eqref{eq:first}. Substituting Eqs.~\eqref{eq:swapping1} and \eqref{eq:swapping2} into the L.H.S.~of Eq.~\eqref{eq:first}, we have
\begin{align}
& \text{Integrand of L.H.S.~of Eq.~\eqref{eq:first}} \nonumber  \\
&\qquad = p(x^{\dagger}) \cdot \biggl\{ p(y^{{\rm (adv)}}(x^{\dagger}) | x^{\dagger}) \cdot f\left(r^{\ast}(x^{\dagger}, y^{{\rm (adv)}}(x^{\dagger})) - \epsilon \cdot r^{\ast} (x^{\dagger} ) p(y^{{\rm (max)}}(x^{\dagger}) | x^{\dagger}) \right)  \nonumber \\
& \qquad  \qquad \qquad   \qquad \qquad + p(y^{{\rm (max)}}(x^{\dagger}) | x^{\dagger}) \cdot f\left(r^{\ast}(x^{\dagger}, y^{{\rm (max)}}(x^{\dagger})) + \epsilon \cdot r^{\ast} (x^{\dagger} ) p(y^{{\rm (adv)}}(x^{\dagger}) | x^{\dagger})\right)  \biggr\} \label{eq:inter}.
\end{align}

Because $f(\cdot)$ is differentiable, we can apply the first order Taylor expansion to the two terms involving $f(\cdot)$ in Eq.~\eqref{eq:inter} as 
\begin{align}
&f\left(r^{\ast}(x^{\dagger}, y^{{\rm (adv)}}(x^{\dagger})) - \epsilon \cdot r^{\ast} (x^{\dagger} ) p(y^{{\rm (max)}}(x^{\dagger}) | x^{\dagger}) \right) \nonumber \\
 & \qquad \qquad = f\left( r^{\ast}(x^{\dagger}, y^{{\rm (adv)}}(x^{\dagger})) \right) -  \epsilon \cdot f^{\prime}\left( r^{\ast}(x^{\dagger}, y^{{\rm (adv)}}(x^{\dagger})) \right) \cdot r^{\ast} (x^{\dagger} ) p(y^{{\rm (max)}}(x^{\dagger}) | x^{\dagger}) + \mathcal{O}(\epsilon^2) \nonumber , \\
 & f\left(r^{\ast}(x^{\dagger}, y^{{\rm (max)}}(x^{\dagger})) + \epsilon \cdot r^{\ast} (x^{\dagger} ) p(y^{{\rm (adv)}}(x^{\dagger}) | x^{\dagger}) \right)  \nonumber \\
&\qquad \qquad = f\left( r^{\ast}(x^{\dagger}, y^{{\rm (max)}}(x^{\dagger})) \right) + \epsilon \cdot f^{\prime}\left( r^{\ast}(x^{\dagger}, y^{{\rm (max)}}(x^{\dagger})) \right) \cdot r^{\ast} (x^{\dagger} ) p(y^{{\rm (adv)}}(x^{\dagger}) | x^{\dagger}) + \mathcal{O}(\epsilon^2) \label{eq:inter2}.
\end{align}

Substituting Eq.~\eqref{eq:inter2} into Eq.~\eqref{eq:inter}, we have
{\small
\begin{align}
& \text{Integrand of L.H.S.~of Eq.~\eqref{eq:first}}   \nonumber \\
& = p(x^{\dagger}, y^{{\rm (adv)}}(x^{\dagger})) \cdot f\left(r^{\ast}(x^{\dagger}, y^{{\rm (adv)}}(x^{\dagger}))\right) + p(x^{\dagger}, y^{{\rm (max)}}(x^{\dagger})) \cdot f\left(r^{\ast}(x^{\dagger}, y^{{\rm (max)}}(x^{\dagger}))\right) \nonumber \\
& \qquad \qquad  + \epsilon \cdot r^{\ast} (x^{\dagger} ) p(y^{{\rm (max)}}(x^{\dagger}) | x^{\dagger}) p(y^{{\rm (adv)}}(x^{\dagger}) | x^{\dagger}) \cdot \left\{ f^{\prime}\left(r^{\ast}(x^{\dagger}, y^{{\rm (max)}}(x^{\dagger}))\right) - f^{\prime}\left(r^{\ast}(x^{\dagger}, y^{{\rm (adv)}}(x^{\dagger}))\right) \right\} + \mathcal{O}(\epsilon^2), \label{eq:abab}
\end{align}}Since $f(\cdot)$ is convex, its derivative $f^{\prime}(\cdot)$ is non-decreasing. Also $r^{\ast}(x^{\dagger}, y^{{\rm (adv)}}(x^{\dagger})) > r^{\ast}(x^{\dagger}, y^{{\rm (max)}}(x^{\dagger}))$ holds because of Eq.~\eqref{eq:bigger}. Therefore, we have $f^{\prime}(r^{\ast}(x^{\dagger}, y^{{\rm (max)}}(x^{\dagger}))) - f^{\prime}(r^{\ast}(x^{\dagger}, y^{{\rm (adv)}}(x^{\dagger}))) \leq 0.$ 

First, assume $f^{\prime}(r^{\ast}(x^{\dagger}, y^{{\rm (max)}}(x^{\dagger}))) - f^{\prime}(r^{\ast}(x^{\dagger}, y^{{\rm (adv)}}(x^{\dagger}))) = 0$. Then, $f(\cdot)$ is exactly linear in the interval of $[r^{\ast}(x^{\dagger}, y^{{\rm (max)}}(x^{\dagger})), r^{\ast}(x^{\dagger}, y^{{\rm (adv)}}(x^{\dagger}))]$; hence, $\mathcal{O}(\epsilon^2)$ term in Eq.~\eqref{eq:abab} is exactly 0 for sufficiently small $\epsilon$.
Hence, we have
\begin{align}
\text{Eq.~\eqref{eq:abab}} & = p(x^{\dagger}, y^{{\rm (adv)}}(x^{\dagger})) \cdot f(r^{\ast}(x^{\dagger}, y^{{\rm (adv)}}(x^{\dagger}))) + p(x^{\dagger}, y^{{\rm (max)}}(x^{\dagger})) \cdot f(r^{\ast}(x^{\dagger}, y^{{\rm (max)}}(x^{\dagger}))) \nonumber \\
&= \text{Integrand of R.H.S.~of Eq.~\eqref{eq:first}}. \label{eq:71}
\end{align}

Next, assume $f^{\prime}(r^{\ast}(x^{\dagger}, y^{{\rm (max)}}(x^{\dagger}))) - f^{\prime}(r^{\ast}(x^{\dagger}, y^{{\rm (adv)}}(x^{\dagger}))) < 0$. In this case, since the coefficient of $\epsilon$ in Eq.~\eqref{eq:ebeb} is negative, there exits sufficiently small $\epsilon > 0$ such that 
\begin{align}
\epsilon \cdot \underbrace{r^{\ast} (x^{\dagger} )}_{>0} \underbrace{p(y^{{\rm (max)}}(x^{\dagger}) | x^{\dagger}) p(y^{{\rm (adv)}}(x^{\dagger}) | x^{\dagger})}_{>0} \cdot \underbrace{\left\{ f^{\prime}\left(r^{\ast}(x^{\dagger}, y^{{\rm (max)}}(x^{\dagger}))\right) - f^{\prime}\left(r^{\ast}(x^{\dagger}, y^{{\rm (adv)}}(x^{\dagger}))\right) \right\}}_{<0} + \mathcal{O}(\epsilon^2) < 0 \label{eq:ebeb}
\end{align}
Thus, we have
\begin{align}
\text{Eq.~\eqref{eq:abab}} &< p(x^{\dagger}, y^{{\rm (adv)}}(x^{\dagger})) \cdot f(r^{\ast}(x^{\dagger}, y^{{\rm (adv)}}(x^{\dagger}))) + p(x^{\dagger}, y^{{\rm (max)}}(x^{\dagger})) \cdot f(r^{\ast}(x^{\dagger}, y^{{\rm (max)}}(x^{\dagger}))) \nonumber \\
&= \text{R.H.S.~of Eq.~\eqref{eq:first}}. \label{eq:72}
\end{align}
In both Eqs.~\eqref{eq:71} and \eqref{eq:72}, by taking an integral in $\mathcal{S}$, Eq.~\eqref{eq:first} holds.

In summary, since Eqs.~\eqref{eq:first}, \eqref{eq:second} and \eqref{eq:third} all hold, the newly constructed, $r_{\rm new}(\cdot, \cdot)$, is still in $\mathcal{U}_f$.

We now show that $r_{\rm new}(\cdot, \cdot)$ actually gives larger objective of Eq.~\eqref{eq:sugi} than $r^{\ast}(\cdot, \cdot)$, which contradicts Eq.~\eqref{eq:sugi}. Since the value of $r_{\rm new}(\cdot, \cdot)$ is mostly the same as that of $r^{\ast}(\cdot, \cdot)$ except that we have Eqs.~\eqref{eq:swapping1} and \eqref{eq:swapping2}, we only need to consider the part they differ. Therefore, it is sufficient to show
{\small
\begin{align}
&\int_{x \in \mathcal{S}} p(x)r^{\ast}(x) \underbrace{\left\{ p(y^{{\rm (adv)}}(x) | x) r_{\rm new}(y^{{\rm (adv)}}(x) | x)  \ell(g^{\rm (adv)}(x), y^{{\rm (adv)}}(x)) + p(y^{{\rm (max)}}(x) | x) r_{\rm new}(y^{{\rm (max)}}(x) | x)  \ell(g^{\rm (adv)}(x), y^{{\rm (max)}}(x)) \right\}}_{\rm (a)} {\rm d}x \nonumber \\
& \ \  >  \int_{x \in \mathcal{S}} p(x)r^{\ast}(x) \underbrace{\left\{p(y^{{\rm (adv)}}(x) | x) r^{\ast}(y^{{\rm (adv)}}(x) | x)  \ell(g^{\rm (adv)}(x), y^{{\rm (adv)}}(x)) + p(y^{{\rm (max)}}(x) | x) r^{\ast}(y^{{\rm (max)}}(x) | x)  \ell(g^{\rm (adv)}(x), y^{{\rm (max)}}(x))\right\}}_{\rm (b)} {\rm d}x \label{eq:toshow}.
\end{align}}Subtracting (b) from (a) in Eq.~\eqref{eq:toshow}, and using Eqs.~\eqref{eq:swapping1} and \eqref{eq:swapping2}, we have
\begin{align}
& p(y^{{\rm (adv)}}(x^{\dagger}) | x^{\dagger})  \ell(g^{\rm (adv)}(x^{\dagger}), y^{{\rm (adv)}}(x^{\dagger})) \left\{ r_{\rm new}(y^{{\rm (adv)}}(x^{\dagger}) | x^{\dagger}) - r^{\ast}(y^{{\rm (adv)}}(x^{\dagger}) | x^{\dagger}) \right\} \nonumber \\ 
& \qquad  + p(y^{{\rm (max)}}(x^{\dagger}) | x^{\dagger})  \ell(g^{\rm (adv)}(x^{\dagger}), y^{{\rm (max)}}(x^{\dagger})) \left\{ r_{\rm new}(y^{{\rm (max)}}(x^{\dagger}) | x^{\dagger}) - r^{\ast}(y^{{\rm (max)}}(x^{\dagger}) | x^{\dagger}) \right\} \nonumber \\
&\qquad \ \ \ \ = \epsilon \cdot p(y^{{\rm (adv)}}(x^{\dagger}) | x^{\dagger}) p(y^{{\rm (max)}}(x^{\dagger}) | x^{\dagger}) \left\{ - \ell(g^{\rm (adv)}(x^{\dagger}), y^{{\rm (adv)}}(x^{\dagger})) + \ell(g^{\rm (adv)}(x^{\dagger}), y^{{\rm (max)}}(x^{\dagger})) \right\}. \label{eq:lastineq}
\end{align} 
We now show $\ell(g^{\rm (adv)}(x^{\dagger}), y^{{\rm (max)}}(x^{\dagger})) > \ell(g^{\rm (adv)}(x^{\dagger}), y^{{\rm (adv)}}(x^{\dagger}))$.
Suppose $\ell(g^{\rm (adv)}(x^{\dagger}), y^{{\rm (max)}}(x^{\dagger})) \leq \ell(g^{\rm (adv)}(x^{\dagger}), y^{{\rm (adv)}}(x^{\dagger}))$. Construct $g^{\prime} \in \mathbb{R}^K$ by swapping the $y^{\rm (max)}$-th and $y^{\rm (adv)}$-th elements of $g^{\rm (adv)}(x^{\dagger})$, while retaining other elements to be exactly the same.
Then, because of the assumption that $\ell(\cdot, \cdot)$ is invariant to class permutation, we have $\ell(g^{\prime}, y^{{\rm (max)}}(x^{\dagger})) = \ell(g^{\rm (adv)}(x^{\dagger}), y^{{\rm (adv)}}(x^{\dagger}))$ and $\ell(g^{\prime}, y^{{\rm (adv)}}(x^{\dagger})) = \ell(g^{\rm (adv)}(x^{\dagger}), y^{{\rm (max)}}(x^{\dagger}))$.
Combining this with Eq.~\eqref{eq:wei}, we have
{\small
\begin{align}
& p(y^{{\rm (adv)}}(x^{\dagger})|x^{\dagger}) r^{\ast}(y^{{\rm (adv)}}(x^{\dagger}) |x^{\dagger}) \cdot \ell(g^{\prime}, y^{{\rm (adv)}}(x^{\dagger})) +  p(y^{{\rm (max)}}(x^{\dagger})|x^{\dagger}) r^{\ast}(y^{{\rm (max)}}(x^{\dagger})|x^{\dagger})\cdot \ell(g^{\prime}, y^{{\rm (max)}}(x^{\dagger})) \nonumber \\
& \ \   < p(y^{{\rm (adv)}}(x^{\dagger})|x^{\dagger}) r^{\ast}(y^{{\rm (adv)}}(x^{\dagger}) |x^{\dagger}) \cdot \ell(g^{\rm (adv)}(x^{\dagger}), y^{{\rm (adv)}}(x^{\dagger})) +  p(y^{{\rm (max)}}(x^{\dagger})|x^{\dagger}) r^{\ast}(y^{{\rm (max)}}(x^{\dagger})|x^{\dagger})\cdot \ell(g^{\rm (adv)}(x^{\dagger}), y^{{\rm (max)}}(x^{\dagger})), \label{eq:gang}
\end{align}}which is in contradiction with the fact that $g^{\rm (adv)}$ achieves the minimal value of Eq.~\eqref{eq:inner}.
Therefore, we have $\ell(g^{\rm (adv)}(x^{\dagger}), y^{{\rm (max)}}(x^{\dagger})) > \ell(g^{\rm (adv)}(x^{\dagger}), y^{{\rm (adv)}}(x^{\dagger}))$. Hence, Eq.~\eqref{eq:lastineq} is positive. 
Multiplying Eq.~\eqref{eq:lastineq} by $p(x) r^{\ast}(x)$ and taking an integral in $\mathcal{S}$, it is still positive because of Eq.~\eqref{eq:assume}. Thus, we have Eq.~\eqref{eq:toshow}.

In conclusion, using $r_{\rm new}(\cdot, \cdot)$ gives larger objective of Eq.~\eqref{eq:sugi} than $r^{\ast}(\cdot, \cdot)$. This contradicts the fact that $r^{\ast}(\cdot, \cdot)$ is the solution of Eq.~\eqref{eq:sugi}.
Therefore, $g^{\rm (adv)}$, which is obtained via ARM in Eq.~\eqref{eq:inner_again2}, is Bayes optimal and coincides with $g^{\ast}$ that is obtained via ordinary RM in Eq.~\eqref{eq:ordinary_rm}. 

So far, we showed that $g^{\rm (adv)}$, which is any solution of ARM, coincides with the Bayes optimal classifier almost surely over $q^{\ast}(x)$.
We now turn our focus to the region 
\begin{align}
\mathcal{S}^{\prime} \equiv \left\{ x \ |\ x \in \mathcal{X}, \ p(x) > 0,\ q^{\ast}(x) = 0 \right\}.
\end{align}

Because $g^{\rm (adv)}$ is chosen from all measurable functions, its function value at $x$ is obtained in a point-wise manner:
\begin{align}
g^{\rm (adv)}(x) = \argmin_{\widehat{y} \in \mathbb{R}^K} \sum_{y \in \mathcal{Y}} p(y | x) r^{\ast}(y | x) \ell(\widehat{y}, y). \label{eq:last}
\end{align}

For $x \in \mathcal{S}^{\prime}$, we immediately have $r^{\ast}(x) = q^{\ast}(x)/p(x) = 0$.
Then, for all $x\in \mathcal{S}^{\prime}$, we have $r^{\ast}(x, y) = 0$ for any $r(y|x), y \in \mathcal{Y}$.
Therefore, for $x \in \mathcal{S}^{\prime}$, we can \emph{virtually} set $r^{\ast}(y|x) = 1$ for all $y \in \mathcal{Y}$.
Substituting this to Eq.~\eqref{eq:last}, we have
\begin{align}
g^{\rm (adv)}(x) = \argmin_{\widehat{y} \in \mathbb{R}^K} \sum_{y \in \mathcal{Y}} p(y | x) \ell(\widehat{y}, y),\ \  \text{for $x \in \mathcal{S}^{\prime}$}. \label{eq:last2}
\end{align}
If follows from Eq.~\eqref{eq:last2} and the use of the classification-calibrated loss that 
\begin{align}
\argmax_y g_y^{\rm (adv)}(x) = \argmax_y p(y|x), \ \ \ \text{for $x \in \mathcal{S}^{\prime}$.} \label{eq:last3}
\end{align}
This particular $g^{\rm (adv)}$ coincides with the Bayes optimal classifier for all $x \in \mathcal{S}^{\prime}$. 
Define
\begin{align}
\mathcal{S}_{\rm diff} &\equiv  \left\{ x \ \middle| \ x \in \mathcal{X},\ p(x) > 0, \  \argmax_{y \in \mathcal{Y}} p(y|x) \neq \argmax_{y \in \mathcal{Y}} g^{({\rm adv})}_{y}(x) \right\}. \nonumber \\
& = \mathcal{S} \cup \underbrace{\left\{ x \ \middle| \ x \in \mathcal{X},\ p(x) > 0, \ q^{\ast}(x) = 0, \  \argmax_{y \in \mathcal{Y}} p(y|x) \neq \argmax_{y \in \mathcal{Y}} g^{({\rm adv})}_{y}(x) \right\}}_{\equiv \mathcal{S}_{\rm diff}^{\rm \prime}}.
\end{align}
Then we have
\begin{align}
\int_{x \in \mathcal{S}_{\rm diff}} p(x) {\rm d}x & = \int_{x \in \mathcal{S}} p(x) {\rm d}x + \underbrace{\int_{x \in \mathcal{S}_{\rm diff}^{\rm \prime}} p(x) {\rm d}x}_{\qquad = 0 \ \ \text{$\because$ Eq.~\eqref{eq:last3}.}} \nonumber \\
& = \int_{x \in \mathcal{S}} \frac{q^{\ast}(x)}{r^{\ast}(x)} {\rm d}x  \nonumber \\
& \leq \frac{1}{ \min_{x \in \mathcal{S}} r^{\ast}(x)} \int_{x \in \mathcal{S}} q^{\ast}(x) {\rm d}x \nonumber \\
& = 0.
\end{align}
Therefore, the particular $g^{\rm (adv)}$ coincides with the Bayes optimal classifier almost surely over $p(x)$.

Finally, we consider binary classification, where the key differences to multi-class classification are that the prediction function is $\mathcal{X} \to \mathbb{R}$, $\mathcal{Y} = \{ +1, -1\}$, and the prediction is made based on the sign of the output of the prediction function.
Therefore, we need to replace all `argmax' in the above proof with `sign'. In addition, to show $\ell(g^{\rm (adv)}(x^{\dagger}), y^{{\rm (max)}}(x^{\dagger})) > \ell(g^{\rm (adv)}(x^{\dagger}), y^{{\rm (adv)}}(x^{\dagger}))$, we construct $g^{\prime}$ in Eq.~\eqref{eq:gang} by $- g^{\rm (adv)}(x^{\dagger})$. With these two modifications, all the arguments for multi-class classification hold for binary classification. \qed

\begin{rem} \label{app:dis_kl}
\upshape
Here, we show that \emph{if the KL divergence is used} for the $f$-divergence, the prediction of \emph{any} solution $g^{\rm (adv)}$ of ARM coincides with that of the Bayes optimal classifier almost surely over $p(x).$

In the following, assume the KL divergence is used.
Given prediction function $g$, the density ratio put by the adversary becomes 
\begin{align}
r^{\ast}(x, y) = \frac{1}{Z(\gamma)}{\rm exp}\left( \frac{\ell(g(x), y)}{\gamma} \right), \label{eq:weihuahaha}
\end{align}
where
\begin{align}
Z(\gamma) = \mathbb{E}_{p(x, y)}\left[{\rm exp}\left( \frac{\ell(g(x), y)}{\gamma} \right) \right],
\end{align}
and $\gamma$ is chosen so that the following equality holds:
\begin{align}
\mathbb{E}_{p(x, y)} \left[ r^{\ast}(x, y) \log r^{\ast}(x, y) \right] = \delta.
\end{align}

From Eq.~\eqref{eq:weihuahaha}, we see that $r^{\ast}(\cdot, \cdot)$ is a positive function for any $g$ as long as $\ell(\cdot, \cdot)$ is bounded.
Thus, $q^{\ast}(x) \equiv \sum_{y \in \mathcal{Y}} r^{\ast}(x, y) p(x, y)$ is also positive for $x \in \mathcal{X}$ such that $p(x) > 0$.
On the other hand, because we assume $q \ll p$, $p(x) = 0$ implies $q(x) = 0$. Thus, we have $q^{\ast}(x) > 0$ iff $p(x) > 0$.

Now, assume that $q^{\ast}(x) > 0$ iff $p(x) > 0$, and $g^{\rm (adv)}$ coincides with the Bayes optimal classifier almost surely over $q^{\ast}(x)$. 
Then, we have
\begin{align}
\int_{x \in \mathcal{S}} q^{\ast}(x) {\rm d}x = 0, \label{eq:assume1}
\end{align}
where 
\begin{align}
\mathcal{S}^{\prime} \equiv  \left\{ x \ \middle| \ x \in \mathcal{X},\ p(x) > 0, \ q^{\ast}(x) > 0,\  \argmax_{y \in \mathcal{Y}} p(y|x) \neq \argmax_{y \in \mathcal{Y}} g^{({\rm adv})}_{y}(x) \right\}.
\end{align}
Because $r^{\ast}(x) \equiv q^{\ast}(x)/p(x)$ is positive, $\epsilon \equiv \min_{x \in \mathcal{S}} r^{\ast}(x)$ is also positive. Then, we have
\begin{align}
\int_{x \in \mathcal{S}} p(x) {\rm d}x &= \int_{x \in \mathcal{S}}  \frac{q^{\ast}(x)}{r^{\ast}(x)} {\rm d}x \nonumber \\
& \leq \frac{1}{\epsilon}\int_{x \in \mathcal{S}}  q^{\ast}(x) {\rm d}x \nonumber \\
& = 0.
\end{align}
Therefore, $g^{\rm (adv)}$ coincides with the Bayes optimal classifier almost surely over the training density $p(x)$.
\end{rem}

\section{Proof of Lemma \ref{lem:steeper_calibrated}} \label{app:proof_steeper_calibrated}
By assumption, $\ell(\widehat{y}, y)$ is a convex margin loss.
Thus, we can let $\ell(\widehat{y}, y) = \phi(y \widehat{y})$, where $\phi(\cdot)$ is a convex function.
On the other hand, by Definition \ref{def:steeper}, for some non-constant, non-decreasing and non-negative function $h(\cdot)$, the steeper loss $\ell_{{\rm steep}}(\widehat{y}, y)$, satisfies 
\begin{align}
 \frac{\partial \ell_{\rm steep}(\widehat{y}, y)}{\partial \widehat{y}} &=  h(\ell(\widehat{y}, y)) \frac{\partial \ell(\widehat{y}, y)}{\partial \widehat{y}} \nonumber \\
  &= h(\phi(y \widehat{y})) \frac{\partial \phi(y \widehat{y})}{\partial \widehat{y}}. \label{eq:heyhey}
\end{align}
From Eq.~\eqref{eq:heyhey}, it is easy to see that $\ell_{\rm steep}(\widehat{y}, y)$ is also a margin loss and can be written as $\ell_{\rm steep}(\widehat{y}, y) = \phi_{\rm steep}(y \widehat{y})$.
Our first goal is to show that $\phi_{\rm steep}(\cdot)$ is convex.
To this end, it is sufficient to show that $\frac{\partial \phi_{\rm steep}(y \widehat{y})}{\partial \widehat{y}} = h(\phi(y \widehat{y})) \frac{\partial \phi(y \widehat{y})}{\partial \widehat{y}}$ is non-decreasing in $\widehat{y}$, $\mathbb{Y} = \{ +1, -1\}$.
 
Since $\phi(y \widehat{y})$ is convex in $\widehat{y}$, $\frac{\phi(y \widehat{y})}{\partial \widehat{y}}$ is non-decreasing in $\widehat{y}$. 
Let $\widehat{y}_{\alpha}$ be the smallest $\widehat{y}_{\alpha}$ such that $\frac{\ell(\widehat{y}_{\alpha}, y_i)}{\partial \widehat{y}} = 0$, if such $\widehat{y}_{\alpha}$ exists.
In the following, we analyze $\phi(y \widehat{y}) \frac{\partial \phi(y \widehat{y})}{\partial \widehat{y}}$, considering two cases: 1) $\widehat{y} \leq \widehat{y}_{\alpha}$ and 2) $\widehat{y}_{\alpha} \leq \widehat{y}$.
Note that $\widehat{y}_{\alpha}$ may not always exist because $\frac{\phi(y \widehat{y})}{\partial \widehat{y}}$ can be negative for any finite $\widehat{y}$, which is the case for the widely-used classification losses such as the exponential loss and the logistic loss. In such a case, we only consider the first case, letting $\widehat{y}_{\alpha}$ arbitrarily large.

{\bf Case 1 $\widehat{y} \leq \widehat{y}_{\alpha}$:}
 By convexity of $\phi(y \widehat{y})$, for $\widehat{y} \leq \widehat{y}_{\alpha}$, $\frac{\phi(y \widehat{y})}{\partial \widehat{y}} \leq 0$ holds and therefore, $\phi(y \widehat{y})$ is non-increasing in $\widehat{y}$. 
Since $h(\cdot)$ is a non-decreasing function, $h(\phi(y \widehat{y}))$ is non-increasing in $\widehat{y}$ for $\widehat{y} \leq \widehat{y}_{\alpha}$. 
In summary, for $\widehat{y} \leq \widehat{y}_{\alpha}$, $\frac{\phi(y \widehat{y})}{\partial \widehat{y}}$ is a non-positive non-decreasing function of $\widehat{y}$, and $h(\phi(y \widehat{y}))$ is a non-negative non-increasing function of $\widehat{y}$. Thus, for $\widehat{y} \leq \widehat{y}_{\alpha}$, their product $h(\phi(y \widehat{y})) \frac{\phi(y \widehat{y})}{\partial \widehat{y}}$ is a non-decreasing function of $\widehat{y}$.

{\bf Case 2 $\widehat{y}_{\alpha} \leq \widehat{y}$:}
By convexity of $\phi(y \widehat{y})$, for $\widehat{y}_{\alpha} \leq \widehat{y}$, $\frac{\phi(y \widehat{y})}{\partial \widehat{y}} \geq 0$ holds and therefore, $\phi(y \widehat{y})$ is non-decreasing in $\widehat{y}$. 
Since $h(\cdot)$ is a non-decreasing function, $h(\phi(y \widehat{y}))$ is non-decreasing in $\widehat{y}$ for $\widehat{y} \leq \widehat{y}_{\alpha}$. 
In summary, for $\widehat{y}_{\alpha} \leq \widehat{y}$, $\frac{\phi(y \widehat{y})}{\partial \widehat{y}}$ is a non-negative non-decreasing function of $\widehat{y}$, and $h(\phi(y \widehat{y}))$ is a non-negative non-decreasing function of $\widehat{y}$. Thus, for $\widehat{y} \leq \widehat{y}_{\alpha}$, their product $h(\phi(y \widehat{y}) )\frac{\phi(y \widehat{y})}{\partial \widehat{y}}$ is a non-decreasing function of $\widehat{y}$.

Therefore, \emph{for any $\widehat{y}$}, $h(\phi(y \widehat{y})) \frac{\phi(y \widehat{y})}{\partial \widehat{y}}$ is a non-decreasing function of $\widehat{y}$, which directly indicates that the steeper loss, $\phi_{{\rm steep}}(y \widehat{y})$, is convex.

We now utilize the fact that a convex margin loss $\psi(y \widehat{y})$ is classification calibrated iff $\psi^{\prime}(0) < 0$ [Theorem 6 in \cite{bartlett2006convexity}]. Using this fact, because $\phi(y \widehat{y})$ is classification calibrated, we have $\phi^{\prime}(0) < 0$.
Furthermore, from the assumption, we have $h(\phi(0)) > 0$. Therefore, we have $\phi_{{\rm steep}}^{\prime}(0) =  h(\phi(0))\phi^{\prime}(0) < 0$. Using the fact again, we immediately have that $\phi_{{\rm steep}}(y \widehat{y})$ is classification calibrated. \qed

\begin{rem}
In the proof, we need to assume $h(\phi(0)) > 0$. From Appendix \ref{app:proof_steeper}, we know that $h(\ell)$ corresponds to the weight put by the adversary to data points with a loss value of $\ell$. We see from Eq.~\eqref{eq:kl1_opt} that when the KL divergence is used, the adversary will only assign positive weights to data losses. Therefore, Lemma \ref{lem:steeper_calibrated} always holds when the KL divergence is used. 
\end{rem}

\section{Proof of Theorem \ref{thm:steeper}} \label{app:proof_steeper}
Let $\theta^{\ast}$ be the stationary point of Eq.~\eqref{eq:uncertain_set_empirical}.
By using a chain rule and Danskin's theorem \cite{danskin67}, $\theta^{\ast}$ satisfies
\begin{align}
\frac{1}{N} \sum_{i = 1}^{N} r^{\ast}_i \left. \frac{\partial \ell(\widehat{y}, y_i)}{\partial \widehat{y}} \right|_{\widehat{y} = g_{\theta^{\ast}}(x_i)} \cdot \left. \nabla_{\theta}g_{\theta}(x_i) \right|_{\theta = \theta^{\ast}}  \in  \vector{0}, \label{eq:a}
\end{align}
where $\vector{r}^{\ast}$ is the solution of inner maximization of Eq.\,\eqref{eq:uncertain_set_empirical} at the stationary point. 

Now, we analyze $\vector{r}^{\ast}$, which is the solution of Eq.~\eqref{eq:uncertain_set_empirical} at the stationary point $\theta^{\ast}$.
For notational convenience, for $1 \leq i \leq N$, let us denote $\ell_i(\theta^{\ast})$ by $\ell_i^{\ast}$.
Then, $\vector{r}^{\ast}$ is the solution of the following optimization problem.
\begin{align}
&\max_{\vector{r} \in \widehat{\mathcal{U}}_{f}} \frac{1}{N}\sum_{i = 1}^{N} r_i \ell_i^{\ast}, 
\label{eq:uncertain_set_empirical_again} \\
&\widehat{\mathcal{U}}_{f}= \left\{ \vector{r} \   \middle|  \ \frac{1}{N}\sum_{i=1}^N f \left( r_i \right) \leq \delta,\  \frac{1}{N}\sum_{i=1}^N r_i = 1,\ \vector{r} \geq 0 \right \}, \label{eq:uncertain_set_empirical2_again}
\end{align}

Note that Eq.~\eqref{eq:uncertain_set_empirical_again} is a convex optimization problem because it has a linear objective with a convex constraint; thus, any local maximum is the global maximum.
Nonetheless, there can be multiple solutions that attain the same global maxima.
Among those solutions, we now show that there exists $\vector{r}^{\ast}$ such that elements of $\vector{r}^{\ast}$ has the monotonic relationship to the corresponding data losses, i.e., for any $1 \leq j \neq k \leq N$,
\begin{align}
\ell_j^{\ast} < \ell_k^{\ast}  \Rightarrow 0 \leq r^{\ast}_j \leq r^{\ast}_{k}, \label{eq:876}\\
\ell_j^{\ast} = \ell_k^{\ast} \Rightarrow 0 \leq r^{\ast}_j = r^{\ast}_{k} \label{eq:678}.
\end{align}
To prove this, we assume we obtain one of optimal solutions of Eq.~\eqref{eq:uncertain_set_empirical_again}, which we denote as $\vector{r}^{\prime \ast}$.
If this $\vector{r}^{\prime \ast}$ satisfies Eqs.~\eqref{eq:876} and \eqref{eq:678} for any $j$ and $k$, then we are done. In the following, we assume $\vector{r}^{\prime \ast}$ does not satisfy either Eqs.~\eqref{eq:876} or \eqref{eq:678}.

First, assume that $\vector{r}^{\prime \ast}$ does not satisfy Eq.~\eqref{eq:876}. Then, there exist $1 \leq j \neq  k \leq N$ such that $\ell_j^{\ast} < \ell_k^{\ast}$ but $r^{\prime \ast}_j > r^{\prime \ast}_{k}$. Define $\vector{r}^{\prime \prime \ast}$ such that 
\begin{align}
r^{\prime \prime \ast}_i = \begin{cases}
r^{\prime \ast}_i  \ \ \ \text{if $i \neq j, k$} \\
r^{\prime \ast}_j  \ \ \ \text{if $i = k$} \\
r^{\prime \ast}_k \ \ \ \text{if $i = j$}
\end{cases} \text{\ \ for $1 \leq i \leq N$}.
\end{align}
Then, it is easy to see $\vector{r}^{\prime \prime \ast} \in \widehat{\mathcal{U}}_f$, and the following holds:
\begin{align}
\frac{1}{N} \sum_{i = 1}^{N} r_{i}^{\prime \prime \ast} \ell_i^{\ast} - \frac{1}{N} \sum_{i = 1}^{N} r_{i}^{\prime \ast} \ell_i^{\ast} &= \frac{1}{N} \left( r_{j}^{\prime \prime \ast} \ell_j^{\ast} + r_{k}^{\prime \prime \ast} \ell_k^{\ast} -  r_{j}^{\prime \ast} \ell_j^{\ast} - r_{k}^{\prime \ast} \ell_k^{\ast}\right) \nonumber \\
& = \frac{1}{N} \left( r_{k}^{\prime \ast} \ell_j^{\ast} + r_{j}^{\prime \ast} \ell_k^{\ast} -  r_{j}^{\prime \ast} \ell_j^{\ast} - r_{k}^{\prime \ast} \ell_k^{\ast}\right) \nonumber \\
& = \frac{1}{N} \left( r^{\prime \ast}_j - r^{\prime \ast}_{k} \right) \left( \ell_k^{\ast} - \ell_j^{\ast}\right) \nonumber \\
& > 0.
\end{align}
Therefore, the newly defined $\vector{r}^{\prime \prime \ast}$ attains the larger objective value of Eq.~\eqref{eq:uncertain_set_empirical_again}, which contradicts the assumption that $\vector{r}^{\prime \ast}$ is the optimal solution of Eq.~\eqref{eq:uncertain_set_empirical_again}. Thus, $\vector{r}^{\prime \ast}$ always satisfies Eq.~\eqref{eq:876}.

Second, assume that $\vector{r}^{\prime \ast}$ does not satisfy Eq.~\eqref{eq:678}. Then, there exist $1 \leq j \neq  k \leq N$ such that $\ell_j^{\ast} = \ell_k^{\ast}$ but $r^{\prime \ast}_j \neq r^{\prime \ast}_{k}$. Define $\vector{r}^{\prime \prime \ast}$ such that 
\begin{align}
r^{\prime \prime \ast}_i = \begin{cases}
r^{\prime \ast}_i  \ \ \ \text{if $i \neq j, k$} \\
(r^{\prime \ast}_i + r^{\prime \ast}_j)/2 \ \ \ \text{if $i = j, k$}
\end{cases} \text{\ \ for $1 \leq i \leq N$}.
\end{align}
Then, it is easy to see $\vector{r}^{\prime \prime \ast} \in \widehat{\mathcal{U}}_f$ because
\begin{align}
\frac{1}{N}\sum_{i=1}^N f \left( r_i^{\prime \prime \ast} \right) &= \frac{1}{N}\left\{ \left(\sum_{i\neq j,k} f \left( r_i^{\prime \prime \ast} \right) \right) +  f \left( r_j^{\prime \prime \ast} \right) + f \left( r_k^{\prime \prime \ast} \right) \right\} \nonumber \\
&= \frac{1}{N}\left\{ \left(\sum_{i\neq j,k} f \left( r_i^{\prime \ast} \right) \right) +  f \left( (r_j^{\prime \ast} + r_k^{\prime \ast})/2  \right) + f \left( (r_j^{\prime \ast} + r_k^{ \prime \ast})/2  \right) \right\} \nonumber \\
& \leq \frac{1}{N}\left\{ \left(\sum_{i\neq j,k} f \left( r_i^{\prime \ast} \right) \right) +  f \left( r_j^{\prime \ast}  \right) + f \left(  r_k^{ \prime \ast}  \right) \right\} \ \ \left(\text{$\because$ convexity of $f(\cdot).$}\right) \nonumber \\
& = \frac{1}{N}\sum_{i=1}^N f \left( r_i^{\prime \ast} \right) \nonumber \\
& \leq \delta. \ \ \ \left(\text{$\because \vector{r}^{\prime \ast} \in \widehat{\mathcal{U}}_f$.}\right) \label{eq:weihua}
\end{align}
Also, it is easy to see that $\vector{r}^{\prime \prime \ast}$ attain the same maximum value as $\vector{r}^{\prime \ast}$; thus, $\vector{r}^{\prime \prime \ast}$ is the optimal solution of Eq.~\eqref{eq:uncertain_set_empirical_again}, and notably, we have $r^{\prime \prime \ast}_j = r^{\prime \prime \ast}_k$ for $\ell_j^{\ast} = \ell_k^{\ast}$. In general, we can start from any $\vector{r}^{\prime \ast} \in \widehat{\mathcal{U}}_f$ and equally distribute the weights to the same losses to obtain $\vector{r}^{\prime \prime \ast}$, which is still in $\widehat{\mathcal{U}}_f$ and attains exactly the same global optimal value in Eq.~\eqref{eq:uncertain_set_empirical_again}.

In the following, we assume we have $\vector{r}^{\ast}$ that satisfies Eqs~\eqref{eq:876} and \eqref{eq:678} for any $1 \leq j \neq k \leq N$.
Then, there exits a non-decreasing non-negative function $r^{\ast}(\cdot): \mathbb{R} \to \mathbb{R}$, such that 
\begin{align}
 r^{\ast}(\ell_i^{\ast}) =  r^{\ast}_i, \ \  \text{for $1 \leq i \leq N$}. \label{eq:eqq}
\end{align}
Let us construct a new loss function $\ell_{\rm DRSL}(\widehat{y}, y)$ by its derivative: 
\begin{align}
 \frac{\partial \ell_{\rm DRSL}(\widehat{y}, y)}{\partial \widehat{y}} \equiv r^{\ast}(\ell(\widehat{y}, y)) \frac{\partial \ell(\widehat{y}, y)}{\partial \widehat{y}}. \label{der}
\end{align}

Then, from Eqs.~\eqref{eq:a}, \eqref{eq:eqq} and \eqref{der}, we immediately have
\begin{align}
\frac{1}{N} \sum_{i = 1}^{N} \left. \frac{\partial \ell_{\rm DRSL}(\widehat{y}, y_i)}{\partial \widehat{y}} \right|_{\widehat{y} = g_{\theta^{\ast}}(x_i)} \cdot \left. \nabla_{\theta}g_{\theta}(x_i) \right|_{\theta = \theta^{\ast}} \in \vector{0}. \label{eq:b}
\end{align}
This readily implies that $\theta^{\ast}$ is a stationary point of Eq.~\eqref{eq:reg_erm}, i.e., ERM using $\ell_{\rm DRSL}(\widehat{y}, y)$.
Furthermore, from Eq.~\eqref{der} and the non-negativeness and non-decreasingness of $r^{\ast}(\cdot)$, we see that the newly constructed loss, $\ell_{\rm DRSL}(\widehat{y}, y)$, is steeper than the original loss, $\ell(\widehat{y}, y)$ (see Definition \ref{def:steeper} for the definition of the steeper loss). Here we see that $h(\cdot)$ in Definition \ref{def:steeper} exactly corresponds to $r^{\ast}(\cdot)$ defined in Eq.~\eqref{eq:eqq}. \qed

\section{Derivation of the Decomposition of the Adversarial Risk} \label{sec:decomposition}
Here, we derive Eq.~\eqref{eq:decompose} for the PE divergence.
\begin{align}
\mathcal{R}_{\rm s\mathchar`-adv}(\theta) - \mathcal{R}(\theta) &\equiv \sup_{w \in \mathcal{W}_{\rm PE}} \mathbb{E}_{p(x, y, z)} \left[ \{w(z) - 1\} \ell(g_{\theta}(x), y) \right] \nonumber \\
& = \sup_{w \in \mathcal{W}_{\rm PE}} \mathbb{E}_{p(z)} \left[ \{w(z) - 1\} \mathbb{E}_{p(x, y | z)} \left[  \ell(g_{\theta}(x), y) \right] \right] \nonumber \\
& = \sup_{w \in \mathcal{W}_{\rm PE}} \sum_{z \in \mathcal{Z}}  p(z) \{w(z) - 1 \} \mathcal{R}_z(\theta). \label{eq:tochu}
\end{align}
It follows from Eq.~\eqref{eq:r_analytical} that for $z \in \mathcal{Z}$, we have the adversarial weight as\footnote{Here, we need to assume that $\delta$ is not so large. Then, we can validly drop the non-negativity inequality constraint of $\widehat{\mathcal{W}}_f$, which is needed to obtain the analytic solution in Eq.~\eqref{eq:r_analytical}.}
\begin{align}
w^{\ast}(z) 
& = \sqrt{\frac{\delta}{\sum_{z^{\prime} \in \mathcal{Z}} p(z^{\prime}) (\mathcal{R}_{z^{\prime}}(\theta) - \mathcal{R}(\theta))^2}} (\mathcal{R}_z(\theta) - \mathcal{R}(\theta)) + 1.
\end{align}
Hence, Eq.~\eqref{eq:tochu} becomes
\begin{align}
\sum_{z \in \mathcal{Z}}  p(z) \{w^{\ast}(z) - 1 \} \mathcal{R}_z(\theta) & = \sqrt{\frac{\delta}{\sum_{z^{\prime} \in \mathcal{Z}} p(z^{\prime}) (\mathcal{R}_{z^{\prime}}(\theta) - \mathcal{R}(\theta))^2}} \sum_{z \in \mathcal{Z}}  p(z) (\mathcal{R}_z(\theta) - \mathcal{R}(\theta)) \mathcal{R}_z(\theta) \nonumber \\
& = \sqrt{\frac{\delta}{\sum_{z^{\prime} \in \mathcal{Z}} p(z^{\prime}) (\mathcal{R}_{z^{\prime}}(\theta) - \mathcal{R}(\theta))^2}} \sum_{z \in \mathcal{Z}}  p(z) (\mathcal{R}_z(\theta) - \mathcal{R}(\theta))^2 \nonumber \\
& = \sqrt{\delta} \cdot \sqrt{\sum_{z \in \mathcal{Z}}  p(z) (\mathcal{R}_z(\theta) - \mathcal{R}(\theta))^2}, 
\end{align}
which concludes our derivation.

\section{Comparison between the Use of Different $f$-divergences} \label{comp}
We qualitatively compare the use of different $f$-divergences.
For $1 \leq x$, the $f$ functions for the PE, KL divergences are $(x-1)^2$, $x \log x$, respectively. 
The function $f$ in Eq.\,\eqref{eq:uncertain_set_empirical_cluster} penalizes the deviation of the adversarial weights from the uniform weights, $\vector{1}_S$.
With the quadratic penalty of the PE divergence, it is hard for the adversary to concentrate large weights onto a small portion of latent categories. 
In contrast, when the KL divergence is used, the adversary tends to put large weights to a small portion of latent categories. 
Hence, users can choose the appropriate divergence depending on their belief on how concentrated the distribution shift occurs.

\section{Formal Statement of the Convergence Rate}
\label{sec:stat-conv-rate}

Denote by $p_z=p(z)$ and $w_z=w(z)$ for $z\in\mathcal{Z}$ and define a set-valued function $\Phi:\mathbb{R}^S\to2^{\mathbb{R}^S}$ as
\begin{align*}\textstyle
\Phi(\vector{u}) = \left\{ \vector{w}\in\mathbb{R}^S \mid
\sum_s(p_s+u_s)f(w_s)\le\delta,\sum_s(p_s+u_s)w_s=1,w_s\ge0 \right\}.
\end{align*}
Then, $\mathcal{W}_f=\Phi(\vector{0})$ and $\widehat{\mathcal{W}}_{f}=\Phi(\vector{u})$ where $u_s=n_s/N-p_s$ for $s=1,\ldots,S$. Similarly, denote by $l_z=\mathbb{E}_{p(x,y)}[p(z\mid x,y)\ell(g_{\theta}(x), y)]$ and define a function $R_\theta:\mathbb{R}^S\to\mathbb{R}$ indexed by $\theta$ as
\begin{align*}\textstyle
R_\theta(\vector{u}')=\sum_sw_s(l_s+u'_s).
\end{align*}
Then, $\mathbb{E}_{p(x,y,z)}[w_z\ell(g_{\theta}(x), y)]=R_\theta(\vector{0})$ and $\widehat{\mathcal{R}}(\vector{w},\theta)=R_\theta(\vector{u}')$ where $u'_s=n_s\overline{{\ell}_s}(\theta)/N-l_s$ for $s=1,\ldots,S$. Finally, the perturbed objective function can be defined by
\begin{align*}\textstyle
J(\theta,\vector{u},\vector{u}') = \sup_{\vector{w}\in\Phi(\vector{u})}R_\theta(\vector{u}')
+\lambda(\vector{u},\vector{u}')\Omega(\theta),
\end{align*}
where the function $\lambda(\vector{u},\vector{u}')\ge0$ serves as the regularization parameter, so that the truly optimal $\theta^*$ is the minimizer of $J(\theta,\vector{0},\vector{0})$ and the empirically optimal $\widehat{\theta}$ is the minimizer of $J(\theta,\vector{u},\vector{u}')$ with the aforementioned perturbations $\vector{u}$ and $\vector{u}'$.

According to the central limit theorem \cite{chung68CPT}, $u_s=\mathcal{O}_p(1/\sqrt{N})$, and $u'_s=\mathcal{O}_p(1/\sqrt{N})$ if the loss $\ell$ is finite. Therefore, we only consider perturbations $\vector{u}$ and $\vector{u}'$ such that $\|\vector{u}\|_2\le\epsilon$ and $\|\vector{u}'\|_2\le\epsilon$ in our analysis, where $0<\epsilon\le\delta/(5\sqrt{S}|f'(1)|)$ is a sufficiently small constant.

We make the following assumptions:
\begin{itemize}
  \item[(a)] $g_\theta(x)$ is linear in $\theta$, and for all $\theta$ under consideration, $\|\nabla_\theta g_{\theta}\|_\infty=\sup_x\|\nabla_\theta g_{\theta}(x)\|_2<\infty$, which implies $\|g_\theta\|_\infty=\sup_x|g_{\theta}(x)|<\infty$;%
 \footnote{This makes $\ell(g_\theta(x),y)$ convex in $\theta$ for all $\ell(t,y)$ convex in $t$ and $\nabla_\theta^2\ell(g_\theta(x),y)$ easy to handle.}
  \item[(b)] $\partial\ell(t,y)/\partial t$ is bounded from below and above for all $t$ such that $|t|\le\|g_\theta\|_\infty$;%
 \footnote{This is a sufficient condition for the Lipschitz continuity of $\ell$. In fact, it must be valid given (a) since $\ell$ is continuously differentiable w.r.t.\ $t$.}
  \item[(c)] $f(t)$ is twice differentiable, and this second derivative is bounded from below by a positive number for all $t$ such that $0\le t\le\sup_{\|\vector{u}\|_2\le\epsilon}\sup_{\vector{w}\in\Phi(\vector{u})}\max_sw_s$;%
 \footnote{This is for the Lipschitz continuity of $J(\theta,\vector{u},\vector{u}')-J(\theta,\vector{0},\vector{0})$. It is satisfied by the KL divergence since $f''(t)=1/t$ and $\Phi(\vector{u})$ is bounded and then $\sup_{\|\vector{u}\|_2\le\epsilon}\sup_{\vector{w}\in\Phi(\vector{u})}\max_sw_s<\infty$, and by the PE divergence since $f''(t)=2$} 
  \item[(d)] $\Omega(\theta)$ is Lipschitz continuous, and $\lambda(\vector{u},\vector{u}')$ converges to $\lambda(\vector{0},\vector{0})$ in $\mathcal{O}(\|\vector{u}\|_2+\|\vector{u}'\|_2)$.
\end{itemize}
We also assume either one of the two conditions holds:
\begin{itemize}
  \item[(e1)] $\Omega(\theta)$ is strongly convex in $\theta$ and $\lambda(\vector{0},\vector{0})>0$;
  \item[(e2)] $\ell(t,y)$ is twice differentiable w.r.t.\ $t$, and $\partial^2\ell(t,y)/\partial t^2$ is lower bounded by a positive number for all $t$ such that $|t|\le\|g_{\theta^*}\|_\infty$. If $t$ is vector-valued, $\partial^2\ell(t,y)/\partial t_i^2$ is lower bounded for all dimensions of $t$ such that $\|t\|_\infty\le\sup_x\|g_{\theta^*}(x)\|_\infty$.%
 \footnote{This makes $J(\theta,\vector{0},\vector{0})$ locally strongly convex in $\theta$ around $\theta^*$. It is satisfied by the logistic loss with the lower bound as $1/(2+\exp(\|g_{\theta^*}\|_\infty)+\exp(-\|g_{\theta^*}\|_\infty))$ and the softmax cross-entropy loss with the lower bound as $\min_y\min_{t_i=\pm\sup_x\|g_{\theta^*}\|_\infty} \exp(t_y)\sum_{i\neq y}\exp(t_i)/(\sum_i\exp(t_i))^2$.}
\end{itemize}

\begin{theorem}[Perturbation analysis]
  \label{thm:pert-analysis}
  Assume (a), (b), (c), (d), and (e1) or (e2). Let $\theta^*$ be the minimizer of $J(\theta,\vector{0},\vector{0})$ and $\theta_{\vector{u},\vector{u}'}$ be the minimizer of $J(\theta,\vector{u},\vector{u}')$. Then, for all $\vector{u}$ and $\vector{u}'$ such that $\|\vector{u}\|_2\le\epsilon$ and $\|\vector{u}'\|_2\le\epsilon$,
  \begin{align*}
  \|\theta_{\vector{u},\vector{u}'}-\theta^*\|_2
  &= \mathcal{O}(\|\vector{u}\|_2^{1/2}+\|\vector{u}'\|_2),\\
  \|J(\theta_{\vector{u},\vector{u}'},\vector{0},\vector{0}) -J(\theta^*,\vector{0},\vector{0})\|_2
  &= \mathcal{O}(\|\vector{u}\|_2^{1/2}+\|\vector{u}'\|_2).
  \end{align*}
\end{theorem}

The convergence rate of the model parameter and the order of the estimation error are immediate corollaries of Theorem~\ref{thm:pert-analysis}.

\begin{theorem}[Convergence rate and estimation error]
  \label{thm:conv-rate-est-err-formal}
  Assume (a), (b), (c), (d), and (e1) or (e2). Let $\theta^*$ be the minimizer of the adversarial expected risk and $\widehat{\theta}_N$ be the minimizer of the adversarial empirical risk given some training data of size $N$. Then, as $N\to\infty$,
  \begin{align*}
  \|\widehat{\theta}_N-\theta^*\|_2
  = \mathcal{O}(N^{-1/4}),
  \end{align*}
  and
  \begin{align*}
\left|\left| \mathcal{R}_{\rm s\mathchar`-adv}(\widehat{\theta}_N) - \mathcal{R}_{\rm s\mathchar`-adv}(\theta^{\ast})\right|\right|_2 = \mathcal{O}(N^{-1/4})
  \end{align*}
\end{theorem}

\section{Proof of the Convergence Rate}
\label{sec:proof-conv-rate}

We begin with the growth condition of $J(\theta,\vector{0},\vector{0})$ at $\theta=\theta^*$.

\begin{lem}[Second-order growth condition]
  \label{thm:growth}
  There exists a constant $C_{J''}>0$ such that
  \begin{align*}
  J(\theta,\vector{0},\vector{0}) \ge J(\theta^*,\vector{0},\vector{0})
  +C_{J''}\|\theta-\theta^*\|_2^2.
  \end{align*}
\end{lem}
\begin{proof}
  First consider the assumption (e1). Let $C_{J''}=(1/2)\lambda(\vector{0},\vector{0})$, so that $J(\theta,\vector{0},\vector{0})$ is strongly convex with parameter $C_{J''}$, i.e.,
  \begin{align*}
  J(\theta,\vector{0},\vector{0}) \ge J(\theta^*,\vector{0},\vector{0})
  +\nabla_\theta J(\theta^*,\vector{0},\vector{0})^\top(\theta-\theta^*)
  +C_{J''}\|\theta-\theta^*\|_2^2.
  \end{align*}
  The lemma follows from the optimality condition which says $\nabla_\theta J(\theta^*,\vector{0},\vector{0})=\vector{0}$.
  
  Second consider the assumption (e2) if (e1) does not hold. Without loss of generality, assume that $\lambda(\vector{0},\vector{0})=0$. Let $\vector{w}^*=\argsup_{\vector{w}\in\Phi(\vector{0})}R_{\theta^*}(\vector{0})$, then according to Danskin's theorem \cite{danskin67},
  \begin{align*}
  \nabla_\theta J(\theta^*,\vector{0},\vector{0})
  &= \mathbb{E}_{p(x,y,z)}[w_z^*\nabla_\theta\ell(g_{\theta^*}(x), y)]\\
  &= \mathbb{E}_{p(x,y,z)}[w_z^*\ell'(g_{\theta^*}(x),y)\nabla_\theta g_{\theta^*}(x)]
  \end{align*}
  where $\ell'(g_{\theta^*}(x),y)$ means $\partial\ell(t, y)/\partial t|_{t=g_{\theta^*}(x)}$. The assumption (a) guarantees that $\nabla_\theta g_{\theta^*}(x)$ is no longer a function of $\theta$, and thus
  \begin{align*}
  \nabla_\theta^2 J(\theta^*,\vector{0},\vector{0})
  &= \mathbb{E}_{p(x,y,z)}[w_z^*\nabla_\theta\ell'(g_{\theta^*}(x),y)\nabla_\theta g_{\theta^*}(x)^\top]\\
  &= \mathbb{E}_{p(x,y,z)}[w_z^*\ell''(g_{\theta^*}(x),y)
  \nabla_\theta g_{\theta^*}(x)\nabla_\theta g_{\theta^*}(x)^\top]
  \end{align*}
  where $\ell''(g_{\theta^*}(x),y)$ means $\partial^2\ell(t, y)/\partial t^2|_{t=g_{\theta^*}(x)}$.
  
  Let $C_{\ell''}=\inf_{|t|\le\|g_{\theta^*}\|_\infty}\min_y\ell''(t,y)$, and by assumption $C_{\ell''}>0$. Also let $C_{\lambda,z}$ be the smallest eigenvalue of $\mathbb{E}_{p(x\mid z)}[\nabla_\theta g_\theta(x)\nabla_\theta g_\theta(x)^\top]$ at $\theta=\theta^*$ for $z\in\mathcal{Z}$. Note that $p(x\mid z)$ generates infinite number of $x$, and $\mathbb{E}_{p(x\mid z)}[\nabla_\theta g_\theta(x)\nabla_\theta g_\theta(x)^\top]$ as an average of infinitely many independent positive semi-definite matrices $\nabla_\theta g_\theta(x)\nabla_\theta g_\theta(x)^\top$ (they are independent as long as $\nabla_\theta g_\theta(x)$ depends on $x$) is positive definite. Thus, $C_{\lambda,z}>0$ for all $z\in\mathcal{Z}$, and subsequently,
  \begin{align*}
  &\hspace{-1em}%
  (\theta-\theta^*)^\top\nabla_\theta^2 J(\theta^*,\vector{0},\vector{0})(\theta-\theta^*)\\
  &\ge \left(\inf_{|t|\le\|g_{\theta^*}\|_\infty}\min_y\ell''(t,y)\right)
  \cdot (\theta-\theta^*)^\top\mathbb{E}_{p(x,y,z)}
  [w_z^*\nabla_\theta g_{\theta^*}(x)\nabla_\theta g_{\theta^*}(x)^\top](\theta-\theta^*)\\
  &= C_{\ell''} \cdot (\theta-\theta^*)^\top\mathbb{E}_{p(x,z)}
  [w_z^*\nabla_\theta g_{\theta^*}(x)\nabla_\theta g_{\theta^*}(x)^\top](\theta-\theta^*)\\
  &= C_{\ell''} \cdot (\theta-\theta^*)^\top\left(\sum_{s=1}^Sp_sw_s^*
  \mathbb{E}_{p(x\mid z=s)}[\nabla_\theta g_{\theta^*}(x)
  \nabla_\theta g_{\theta^*}(x)^\top]\right)(\theta-\theta^*)\\
  &= C_{\ell''} \left(\sum_{s=1}^Sp_sw_s^*(\theta-\theta^*)^\top
  \mathbb{E}_{p(x\mid z=s)}[\nabla_\theta g_{\theta^*}(x)
  \nabla_\theta g_{\theta^*}(x)^\top](\theta-\theta^*)\right)\\
  &\ge C_{\ell''} \left(\sum_{s=1}^Sp_sw_s^*C_{\lambda,s}\|\theta-\theta^*\|_2^2\right)\\
  &\ge C_{\ell''}\min_sC_{\lambda,s}\left(\sum_{s=1}^Sp_sw_s^*\right)\|\theta-\theta^*\|_2^2\\
  &= C_{\ell''}\min_sC_{\lambda,s}\|\theta-\theta^*\|_2^2.
  \end{align*}
  This completes the proof by letting $C_{J''}=C_{\ell''}\min_sC_{\lambda,s}$.
\end{proof}

We then study the Lipschitz continuity of $J(\theta,\vector{u},\vector{u}')$.
\begin{lem}[Lipschitz continuity of the perturbed objective]
  \label{thm:lipschitz-pert}
  For all $\vector{u}$ and $\vector{u}'$ such that $\|\vector{u}\|_2\le\epsilon$ and $\|\vector{u}'\|_2\le\epsilon$, $J(\theta,\vector{u},\vector{u}')$ is Lipschitz continuous with a (not necessarily the best) Lipschitz constant independent of $\vector{u}$ and $\vector{u}'$.
\end{lem}
\begin{proof}
  Define $F(\theta,\vector{u},\vector{u}')=\sup_{\vector{w}\in\Phi(\vector{u})}R_\theta(\vector{u}')$ and let $\vector{w}^*=\argsup_{\vector{w}\in\Phi(\vector{u})}R_\theta(\vector{u}')$. According to Danskin's theorem \cite{danskin67}, $\nabla_\theta F(\theta,\vector{u},\vector{u}')=\sum_sw_s^*\nabla_\theta l_s$ where
  \begin{align*}
  \nabla_\theta l_s
  = \mathbb{E}_{p(x,y)}[p(z=s\mid x,y)\ell'(g_\theta(x),y)\nabla_\theta g_{\theta}(x)].
  \end{align*}
  The assumptions (a) and (b) say that $\|\nabla_\theta g_{\theta}\|_\infty<\infty$ and $|\ell'(g_\theta(x),y)|<\infty$ so that
  \begin{align*}
  \|\nabla_\theta l_s\|_2
  &\le \|\nabla_\theta g_{\theta}\|_\infty
  \left(\sup_{|t|\le\|g_\theta\|_\infty}\max_y|\ell'(t,y)|\right)
  \mathbb{E}_{p(x,y)}[p(z=s\mid x,y)]\\
  &= \|\nabla_\theta g_{\theta}\|_\infty
  \left(\sup_{|t|\le\|g_\theta\|_\infty}\max_y|\ell'(t,y)|\right) p_s\\
  &< \infty,
  \end{align*}
  and it is clear that $w_s^*<\infty$. Hence,
  \begin{align*}
  \|\nabla_\theta F(\theta,\vector{u},\vector{u}')\|_2
  \le \sum_{s=1}^Sw_s^*\|\nabla_\theta l_s\|_2
  < \infty,
  \end{align*}
  which means $F(\theta,\vector{u},\vector{u}')$ is Lipschitz continuous with a Lipschitz constant independent of $\vector{u}$ and $\vector{u}'$.
  
  By the assumption (d), $\Omega(\theta)$ is Lipschitz continuous and there exists a constant $C_\lambda>0$ such that
  \begin{align*}
  \lambda(\vector{u},\vector{u}')
  &\le \lambda(\vector{0},\vector{0})+C_\lambda(\|\vector{u}\|_2+\|\vector{u}'\|_2)\\
  &\le \lambda(\vector{0},\vector{0})+2C_\lambda\epsilon\\
  &< \infty.
  \end{align*}
  As a result, $\lambda(\vector{u},\vector{u}')\Omega(\theta)$ possesses a Lipschitz constant independent of $\vector{u}$ and $\vector{u}'$ as well.
\end{proof}

From now on, we investigate the Lipschitz continuity of the difference function
\begin{align*}
D(\theta)=J(\theta,\vector{u},\vector{u}')-J(\theta,\vector{0},\vector{0}),
\end{align*}
which is the most challenging task in our perturbation analysis. Define
\begin{align*}
D_1(\theta) &= F(\theta,\vector{u},\vector{u}')-F(\theta,\vector{u},\vector{0}),\\
D_2(\theta) &= F(\theta,\vector{u},\vector{0})-F(\theta,\vector{0},\vector{0}),
\end{align*}
where $F(\theta,\vector{u},\vector{u}')=\sup_{\vector{w}\in\Phi(\vector{u})}R_\theta(\vector{u}')$ defined in Lemma~\ref{thm:lipschitz-pert}, and then $D(\theta)$ can be decomposed as
\begin{align*}
D(\theta) = D_1(\theta) + D_2(\theta)
+ (\lambda(\vector{u},\vector{u}')-\lambda(\vector{0},\vector{0}))\Omega(\theta).
\end{align*}
Given the assumption (d), the third function $(\lambda(\vector{u},\vector{u}')-\lambda(\vector{0},\vector{0}))\Omega(\theta)$ is Lipschitz continuous with a Lipschitz constant of order $\mathcal{O}(\|\vector{u}\|_2+\|\vector{u}'\|_2)$. We are going to prove the same property for $D_1(\theta)$ and $D_2(\theta)$ using the assumptions (a), (b) and (c).

\begin{lem}[Lipschitz continuity of the difference function, I]
  \label{thm:lipschitz-diff-part1}
  For any fixed $\vector{u}$ and all $\vector{u}'$ such that $\|\vector{u}'\|_2\le\epsilon$, $D_1(\theta)$ is Lipschitz continuous with a Lipschitz constant of order $\mathcal{O}(\|\vector{u}'\|_2)$.
\end{lem}
\begin{proof}
  According to the chain rule in calculus,
  \begin{align*}
  \|\nabla_\theta D_1(\theta)\|_2
  &= \left\|\sum_{s=1}^S\frac{\partial D_1(\theta)}{\partial l_s}\nabla_\theta l_s\right\|_2\\
  &\le \left|\sum_{s=1}^S\frac{\partial D_1(\theta)}{\partial l_s}\right|
  \cdot \max_s\|\nabla_\theta l_s\|_2\\
  &= \mathcal{O}\left(\left|\sum_{s=1}^S\frac{\partial D_1(\theta)}{\partial l_s}\right|\right),
  \end{align*}
  since we have proven that $\|\nabla_\theta l_s\|_2<\infty$ given the assumptions (a) and (b) in Lemma~\ref{thm:lipschitz-pert}.
  
  By definition,
  \begin{align*}
  D_1(\theta)
  &= \sup_{\vector{w}\in\Phi(\vector{u})}R_\theta(\vector{u}')
  - \sup_{\vector{w}\in\Phi(\vector{u})}R_\theta(\vector{0})\\
  &= \sup_{\vector{w}\in\Phi(\vector{u})}\sum_{s=1}^Sw_s(l_s+u'_s)
  - \sup_{\vector{w}\in\Phi(\vector{u})}\sum_{s=1}^Sw_sl_s.
  \end{align*}
  Let $\vector{w}^*=\argsup_{\vector{w}\in\Phi(\vector{u})}\sum_sw_sl_s$ and $\vector{v}^*=\argsup_{\vector{w}\in\Phi(\vector{u})}\sum_sw_s(l_s+u'_s)$, then according to Danskin's theorem \cite{danskin67}, $\partial D_1(\theta)/\partial l_s=v_s^*-w_s^*$ and
  \begin{align*}
  \left|\sum_{s=1}^S\frac{\partial D_1(\theta)}{\partial l_s}\right|
  &\le \sum_{s=1}^S|v_s^*-w_s^*|\\
  &\le \sqrt{S}\|\vector{v}^*-\vector{w}^*\|_2,
  \end{align*}
  which means $\mathcal{O}(\|\nabla_\theta D_1(\theta)\|_2)=\mathcal{O}(\|\vector{v}^*-\vector{w}^*\|_2)$.
  
  Consider the perturbation analysis of the following optimization problem
  \begin{align}
  \label{eq:mini-pert-opt-1}
  \min_{\vector{w}} \; -\sum_{s=1}^Sw_s(l_s+u'_s) \quad \mathrm{s.t.} \; \vector{w}\in\Phi(\vector{u}),
  \end{align}
  whose objective is perturbed and feasible region is unperturbed. Let
  \begin{align*}
  L(\vector{w},\alpha,\alpha',\vector{u}')
  &= -\sum_{s=1}^Sw_s(l_s+u'_s)
  +\alpha\left(\sum_{s=1}^S(p_s+u_s)f(w_s)-\delta\right)\\
  &\quad + \alpha'\left(\sum_{s=1}^S(p_s+u_s)w_s-1\right)
  \end{align*}
  be the Lagrangian function, where $\alpha\ge0$ and $\alpha'$ are Lagrange multipliers, and for simplicity the nonnegative constraints are omitted. Note that given the assumption (c), if $\alpha\neq0$,
  \begin{align*}
  \frac{\partial^2}{\partial w_i\partial w_j}L(\vector{w},\alpha,\alpha',\vector{u}')
  =\begin{cases}
  \alpha f''(w_i)>0, & i=j,\\
  0, & i\neq j,\\
  \end{cases}
  \end{align*}
  namely, $L(\vector{w},\alpha,\alpha',\vector{u}')$ is locally strongly convex in $\vector{w}$. Thus,
  \begin{itemize}
    \item if $\alpha^*>0$, the second-order sufficient condition (see Definition 6.2 in \citep{bonnans98}) holds at $\vector{w}^*$ that implies the corresponding second-order growth condition according to Theorem 6.3 in \citep{bonnans98};
    \item if $\alpha^*=0$, \eqref{eq:mini-pert-opt-1} is locally a standard linear programming around $\vector{w}^*$ and it is fairly easy to see $\|\vector{v}^*-\vector{w}^*\|_2=\mathcal{O}(\|\vector{u}'\|_2)$ according to Theorem 1 in \citep{robinson77}.
  \end{itemize}
  In the former case, it is obvious that for \eqref{eq:mini-pert-opt-1},
  \begin{itemize}
    \item the objective $-\sum_sw_s(l_s+u'_s)$ is Lipschitz continuous with a  Lipschitz constant $\|\vector{l}\|_2+\epsilon$ independent of $\vector{u}'$;
    \item the difference function $-\sum_sw_su'_s$ is Lipschitz continuous with a Lipschitz constant of order $\mathcal{O}(\|\vector{u}'\|_2)$.
  \end{itemize}
  Therefore, $\|\vector{v}^*-\vector{w}^*\|_2=\mathcal{O}(\|\vector{u}'\|_2)$ by applying Proposition 6.1 in \citep{bonnans98}.
\end{proof}

In order to prove the same property for $D_2(\theta)$, we need several lemmas.

\begin{lem}
  \label{thm:cosine-f'}
  Denote by $f'(\vector{w})=(f'(w_1),\ldots,f'(w_S))^\top$. There exists a constant $C_{\cos}>0$, such that $\cos(\vector{d}\circ f'(\vector{w}),\vector{d})\le1-C_{\cos}$ for all $\vector{u}$ satisfying $\|\vector{u}\|_2\le\epsilon$, $\vector{w}$ satisfying $\sum_s(p_s+u_s)f(w_s)=\delta$ and $\sum_s(p_s+u_s)w_s=1$, and $\vector{d}>\vector{0}$.
\end{lem}
\begin{proof}
  Suppose the lemma is false, i.e., for any sufficiently large $n$, there exists some $\vector{w}_n$ such that $\cos(\vector{d}\circ f'(\vector{w}_n),\vector{d})=1-1/(2n^2)$. Let $\zeta_n=\|\vector{d}\circ f'(\vector{w}_n)\|_2$ and $\eta_n=\zeta_n/\|\vector{d}\|_2$, then
  \begin{align*}
  \|\vector{d}\circ f'(\vector{w}_n)-\eta_n\vector{d}\|_2^2
  &= \|\vector{d}\circ f'(\vector{w}_n)\|_2^2+\eta_n^2\|\vector{d}\|_2^2
  -2\eta_n(\vector{d}\circ f'(\vector{w}_n))^\top\vector{d}\\
  &= 2\zeta_n^2 -2\eta_n\cos(\vector{d}\circ f'(\vector{w}_n),\vector{d})
  \|\vector{d}\circ f'(\vector{w}_n)\|_2\|\vector{d}\|_2\\
  &= 2\zeta_n^2 -2(1-1/(2n^2))\zeta_n^2\\
  &= \zeta_n^2/n^2.
  \end{align*}
  In other words, for $s=1,\ldots,S$,
  \begin{align*}
  |f'(w_{n,s})-\eta_n|
  &= |d_sf'(w_{n,s})-\eta_nd_s|/d_s\\
  &\le \|\vector{d}\circ f'(\vector{w}_n)-\eta_n\vector{d}\|_2/d_s\\
  &= (\zeta_n/n)/d_s\\
  &\le \zeta'_n/n,
  \end{align*}
  where $\zeta'_n=\zeta_n/\min_sd_s$. Consequently, for any $1\le i,j\le S$ and $i\neq j$,
  \begin{align*}
  |f'(w_{n,i})-f'(w_{n,j})|
  &\le |f'(w_{n,i})-\eta_n|+|f'(w_{n,j})-\eta_n|\\
  &\le 2\zeta'_n/n.
  \end{align*}
  Let $C_{f''}>0$ be the lower bound of $f''(t)$ mentioned in the assumption (c). This assumption also guarantees that $f'(t)$ is continuous, and by the mean value theorem, there is some $t$ between $w_{n,i}$ and $w_{n,j}$ such that
  \begin{align*}
  |w_{n,i}-w_{n,j}|
  &= \left|\frac{f'(w_{n,i})-f'(w_{n,j})}{f''(t)}\right|\\
  &\le 2\zeta'_n/(C_{f''}n).
  \end{align*}
  Let $\eta'_n=\sum_sw_{n,s}/S$, then $|w_{n,s}-\eta'_n|\le2\zeta'_n/(C_{f''}n)$ for $s=1,\ldots,S$.
  
  Recall that $\sum_s(p_s+u_s)w_{n,s}=1$, then
  \begin{align*}
  \left(1+\sum_{s=1}^Su_s\right)\eta'_n-1
  &= \sum_{s=1}^S(p_s+u_s)\eta'_n -\sum_{s=1}^S(p_s+u_s)w_{n,s}\\
  &= \sum_{s=1}^S(p_s+u_s)(\eta'_n-w_{n,s}),
  \end{align*}
  and hence
  \begin{align*}
  \left|\left(1+\sum_{s=1}^Su_s\right)\eta'_n-1\right|
  &\le \frac{2\zeta'_n}{C_{f''}n}\sum_{s=1}^S(p_s+u_s)\\
  &= \frac{2\zeta'_n}{C_{f''}n}\left(1+\sum_{s=1}^Su_s\right).
  \end{align*}
  This ensures $|\eta'_n-1/(1+\sum_su_s)|\le2\zeta'_n/(C_{f''}n)$ and implies $|w_{n,s}-1/(1+\sum_su_s)|\le4\zeta'_n/(C_{f''}n)$ for $s=1,\ldots,S$. Since $f(t)$ is twice differentiable and $\|\vector{w}_n\|_2<\infty$, we must have $\zeta'_n<\infty$ and then $\lim_{n\to\infty}w_{n,s}=1/(1+\sum_su_s)$ for all $s=1,\ldots,S$.

  The Taylor expansion of $f(t)$ at $t=1$ is $f(t)=f'(1)(t-1)+\mathcal{O}((t-1)^2)$ since $f(1)=0$, and if $t=1/(1+\sum_su_s)$,
  \begin{align*}
  f\left(\frac{1}{1+\sum_{s=1}^Su_s}\right)
  = -f'(1)\cdot\frac{\sum_{s=1}^Su_s}{1+\sum_{s=1}^Su_s}
  + \mathcal{O}\left(\left(\frac{\sum_{s=1}^Su_s}{1+\sum_{s=1}^Su_s}\right)^2\right).
  \end{align*}
  When $\|\vector{u}\|_2\le\epsilon$, $|\sum_su_s|\le\|\vector{u}\|_1\le\sqrt{S}\epsilon$, and
  \begin{align*}
  f\left(\frac{1}{1+\sum_{s=1}^Su_s}\right)
  &\le |f'(1)|\frac{\sqrt{S}\epsilon}{1-\sqrt{S}\epsilon} + \mathcal{O}(\epsilon^2)\\
  &\le 2\sqrt{S}|f'(1)|\epsilon.
  \end{align*}
  As a result,
  \begin{align*}
  \lim_{n\to\infty}\sum_{s=1}^S(p_s+u_s)f(w_{n,s})
  &= \sum_{s=1}^S(p_s+u_s)f\left(\frac{1}{1+\sum_{s=1}^Su_s}\right)\\
  &= \left(1+\sum_{s=1}^Su_s\right)f\left(\frac{1}{1+\sum_{s=1}^Su_s}\right)\\
  &\le (1+\sqrt{S}\epsilon) \cdot 2\sqrt{S}|f'(1)|\epsilon\\
  &\le 4\sqrt{S}|f'(1)|\epsilon.
  \end{align*}
  However, this is impossible since $\sum_s(p_s+u_s)f(w_{n,s})=\delta\ge5\sqrt{S}|f'(1)|\epsilon$.
\end{proof}

Based on Lemma~\ref{thm:cosine-f'}, we derive the convergence rate of $\Phi(\vector{u})$ to $\Phi(\vector{0})$.

\begin{lem}
  \label{thm:hausdorff}
  Let $d_\mathcal{H}(V,W)$ be the Hausdorff distance between two sets $V$ and $W$:
  \begin{align*}
  d_\mathcal{H}(V,W)
  = \max\left\{\sup_{\vector{v}\in V}\inf_{\vector{w}\in W}\|\vector{v}-\vector{w}\|_2,
  \sup_{\vector{w}\in W}\inf_{\vector{v}\in V}\|\vector{v}-\vector{w}\|_2\right\}.
  \end{align*}
  Then $d_\mathcal{H}(\Phi(\vector{u}),\Phi(\vector{0}))=\mathcal{O}(\|\vector{u}\|_2)$ for all $\vector{u}$ satisfying $\|\vector{u}\|_2\le\epsilon$.
\end{lem}
\begin{proof}
  We are going to prove $\sup_{\vector{w}\in\Phi(\vector{0})}\inf_{\vector{v}\in\Phi(\vector{u})}\|\vector{v}-\vector{w}\|_2 =\mathcal{O}(\|\vector{u}\|_2)$, and the other direction can be proven similarly.
  
  Pick an arbitrary $\vector{w}_0\in\Phi(\vector{0})$. Let $\beta=\delta/(\delta+\|f(\vector{w}_0)\|_2\|\vector{u}\|_2)$ and consider $\vector{v}_1=\beta\vector{w}_0+(1-\beta)\vector{1}$,
  \begin{align*}
  \|\vector{v}_1-\vector{w}_0\|_2
  &= \|(\beta-1)\vector{w}_0+(1-\beta)\vector{1}\|_2\\
  &= (1-\beta)\|\vector{w}_0-\vector{1}\|_2\\
  &\le \|f(\vector{w}_0)\|_2\|\vector{u}\|_2\|\vector{w}_0-\vector{1}\|_2/\delta\\
  &= \mathcal{O}(\|\vector{u}\|_2).
  \end{align*}
  Moreover,
  \begin{align*}
  \sum_{s=1}^S(p_s+u_s)f(v_{1,s})
  &= \sum_{s=1}^S(p_s+u_s)f(\beta w_{0,s}+(1-\beta))\\
  &\le \sum_{s=1}^S(p_s+u_s)(\beta f(w_{0,s})+(1-\beta)f(1))\\
  &= \beta\left(\sum_{s=1}^Sp_sf(w_{0,s})+\sum_{s=1}^Su_sf(w_{0,s})\right)\\
  &\le \beta(\delta+\|f(\vector{w}_0)\|_2\|\vector{u}\|_2)\\
  &= \delta,
  \end{align*}
  where the second line is due to the convexity of $f(t)$, the third line is because $f(1)=0$, and the fourth line is according to Jensen's inequality. This means $\vector{v}_1$ belongs to the set $V_1=\{\vector{w}\in\mathbb{R}^S \mid
  \sum_s(p_s+u_s)f(w_s)\le\delta,w_s\ge0\}$.
  
  However, $\vector{v}_1$ does not belong to the set $V_2=\{\vector{w}\in\mathbb{R}^S \mid
  \sum_s(p_s+u_s)w_s=1,w_s\ge0\}$. Since $V_2$ is a hyperplane, we can easily project $\vector{v}_1$ onto $V_2$ to obtain $\vector{v}_2$, and
  \begin{align*}
  \|\vector{v}_2-\vector{v}_1\|_2
  &= \frac{1}{\|\vector{p}+\vector{u}\|_2}\left|\sum_{s=1}^S(p_s+u_s)v_{1,s}-1\right|\\
  &\le \frac{2}{\|\vector{p}\|_2}\left|\sum_{s=1}^S(p_s+u_s)(\beta w_{0,s}+1-\beta)-1\right|\\
  &\le \sqrt{4S}\left|\beta\sum_{s=1}^Sp_sw_{0,s} +\beta\sum_{s=1}^Su_sw_{0,s}
  +(1-\beta)\sum_{s=1}^Sp_s+(1-\beta)\sum_{s=1}^Su_s-1\right|\\
  &= \sqrt{4S}\left|\beta+\beta\sum_{s=1}^Su_sw_{0,s}
  +(1-\beta)+(1-\beta)\sum_{s=1}^Su_s-1\right|\\
  &\le \sqrt{4S}\beta\left|\sum_{s=1}^Su_sw_{0,s}\right| +\mathcal{O}(\|\vector{u}\|_2^2)\\
  &= \mathcal{O}(\|\vector{u}\|_2).
  \end{align*}
  After this projection, $\vector{v}_2\not\in V_1$ again.
  
  Let $\vector{v}_3$ be the projection of $\vector{v}_2$ onto $\Phi(\vector{u})=V_1\cap V_2$, $T_wV_1(\vector{v}_3)$ be the tangent hyperplane to $V_1$ at $\vector{v}_3$ of $S-1$ dimensions, and $T_w(V_1\cap V_2)(\vector{v}_3)$ be that to $V_1\cap V_2$ at $\vector{v}_3$ of $S-2$ dimensions. As a consequence, $\vector{v}_3-\vector{v}_2\in V_2$  is one of normal vectors to $T_w(V_1\cap V_2)(\vector{v}_3)$ at $\vector{v}_3$, which is also the projection of the normal vector to $T_wV_1(\vector{v}_3)$ at $\vector{v}_3$ onto $V_2$. This means the normal vector to $T_wV_1(\vector{v}_3)$ at $\vector{v}_3$ belongs to the $2$-dimensional plane determined by $\vector{v}_3-\vector{v}_2$ and $\vector{v}_1-\vector{v}_2$, since the latter is a normal vector to $V_2$ at $\vector{v}_2$.
  
  Consider the triangle $(\vector{v}_1-\vector{v}_2,\vector{v}_3-\vector{v}_2,\vector{v}_1-\vector{v}_3)$. This is a right-angled triangle since $\vector{v}_1-\vector{v}_2\perp V_2$ and $\vector{v}_3-\vector{v}_2\in V_2$, so that
  \begin{align*}
  \|\vector{v}_1-\vector{v}_3\|_2 = \frac{\|\vector{v}_1-\vector{v}_2\|_2} {\sin(\vector{v}_2-\vector{v}_3,\vector{v}_1-\vector{v}_3)}.
  \end{align*}
  Subsequently, let $\vector{v}_4$ be the intersection of $\vector{v}_1-\vector{v}_2$ and $T_wV_1(\vector{v}_3)$, due to the convexity of $V_1$,
  \begin{align*}
  \sin^2(\vector{v}_2-\vector{v}_3,\vector{v}_1-\vector{v}_3)
  &\ge \sin^2(\vector{v}_2-\vector{v}_3,\vector{v}_4-\vector{v}_3)\\
  &= 1-\cos^2(\vector{v}_2-\vector{v}_3,\vector{v}_4-\vector{v}_3)\\
  &= 1-\cos^2(\vector{p}+\vector{u},(\vector{p}+\vector{u})\circ f'(\vector{v}_3)),
  \end{align*}
  where $\vector{p}+\vector{u}$ is a normal vector to $V_2$ containing $\vector{v}_2-\vector{v}_3$, $(\vector{p}+\vector{u})\circ f'(\vector{v}_3)$ is a normal vector to $T_wV_1(\vector{v}_3)$ containing $\vector{v}_4-\vector{v}_3$, and both of them belong to the $2$-dimensional plane determined by $\vector{v}_3-\vector{v}_2$ and $\vector{v}_1-\vector{v}_2$ as we have proven above. By definition, $\vector{v}_3$ is on the boundary of $V_1$ such that $\sum_s(p_s+u_s)f(v_{3,s})=\delta$, and according to Lemma~\ref{thm:cosine-f'},
  \begin{align*}
  1-\cos^2(\vector{p}+\vector{u},(\vector{p}+\vector{u})\circ f'(\vector{v}_3))
  &\ge 1-(1-C_{\cos})^2\\
  &= C_{\cos}(2-C_{\cos}),
  \end{align*}
  which implies
  \begin{align*}
  \|\vector{v}_1-\vector{v}_3\|_2
  &\le \frac{\|\vector{v}_1-\vector{v}_2\|_2}{\sqrt{C_{\cos}(2-C_{\cos})}}\\
  &= \mathcal{O}(\|\vector{v}_1-\vector{v}_2\|_2)\\
  &= \mathcal{O}(\|\vector{u}\|_2).
  \end{align*}

  Combining $\|\vector{v}_1-\vector{w}_0\|_2=\mathcal{O}(\|\vector{u}\|_2)$ and $\|\vector{v}_1-\vector{v}_3\|_2=\mathcal{O}(\|\vector{u}\|_2)$ gives us $\|\vector{v}_3-\vector{w}_0\|_2=\mathcal{O}(\|\vector{u}\|_2)$, and thus $\inf_{\vector{v}\in\Phi(\vector{u})}\|\vector{v}-\vector{w}_0\|_2 \le \|\vector{v}_3-\vector{w}_0\|_2 = \mathcal{O}(\|\vector{u}\|_2)$. Since $\vector{w}_0$ is arbitrarily picked from $\Phi(\vector{0})$, the proof is completed.
\end{proof}

\begin{lem}[Lipschitz continuity of the difference function, II]
  \label{thm:lipschitz-diff-part2}
  For all $\vector{u}$ such that $\|\vector{u}\|_2\le\epsilon$, $D_2(\theta)$ is Lipschitz continuous with a Lipschitz constant of order $\mathcal{O}(\|\vector{u}\|_2^{1/2})$.
\end{lem}
\begin{proof}
  The proof goes along the same line with Lemma~\ref{thm:lipschitz-diff-part1}. Let $\vector{w}^*=\argsup_{\vector{w}\in\Phi(\vector{0})}\sum_sw_sl_s$ and $\vector{v}^*=\argsup_{\vector{w}\in\Phi(\vector{u})}\sum_sw_sl_s$, and consider the perturbation analysis of the following optimization problem
  \begin{align*}
    \min_{\vector{w}} \; -\sum_{s=1}^Sw_sl_s \quad \mathrm{s.t.} \; \vector{w}\in\Phi(\vector{u}),
  \end{align*}
  whose objective is unperturbed and feasible region is perturbed. According to Lemma~\ref{thm:hausdorff}, we have $d_\mathcal{H}(\Phi(\vector{u}),\Phi(\vector{0}))=\mathcal{O}(\|\vector{u}\|_2)$, which ensures that the multifunction $\vector{u}\mapsto\Phi(\vector{u})$ is upper Lipschitz continuous and that $d_\mathcal{H}(\{\vector{w}^*\},\Phi(\vector{u}))=\mathcal{O}(\|\vector{u}\|_2)$. Hence, $\|\vector{v}^*-\vector{w}^*\|_2=\mathcal{O}(\|\vector{u}\|_2^{1/2})$ by applying Proposition 6.4 in \citep{bonnans98}.
\end{proof}

Let us summarize what we have obtained so far:
\begin{itemize}
  \item a second-order growth condition of $J(\theta,\vector{0},\vector{0})$ at $\theta=\theta^*$;
  \item the Lipschitz continuity of $J(\theta,\vector{u},\vector{u}')$ with a Lipschitz constant independent of $\vector{u}$ and $\vector{u}'$;
  \item the Lipschitz continuity of $D(\theta)$ with a Lipschitz constant of order $\mathcal{O}(\|\vector{u}\|_2^{1/2}+\|\vector{u}'\|_2)$.
\end{itemize}
Note that $\theta$ is unconstrained, by applying Proposition 6.1 in \citep{bonnans98}, we can obtain
\begin{align*}
\|\theta_{\vector{u},\vector{u}'}-\theta^*\|_2
= \mathcal{O}(\|\vector{u}\|_2^{1/2}+\|\vector{u}'\|_2).
\end{align*}
This immediately implies
\begin{align*}
\|J(\theta_{\vector{u},\vector{u}'},\vector{0},\vector{0})
-J(\theta^*,\vector{0},\vector{0})\|_2
= \mathcal{O}(\|\vector{u}\|_2^{1/2}+\|\vector{u}'\|_2),
\end{align*}
due to the Lipschitz continuity of $J(\theta,\vector{0},\vector{0})$. \qed

\section{Datasets} \label{app:dataset}
\begin{table}[t] 
\begin{center} 
\caption{Summary of dataset statistics.}
\begin{tabular}{|l||c|c|c|}\hline
\label{table:statistics}
Dataset & \#Points & \#classes & Dimension  \\ \hline \hline
blood & 748 & 2 & 4\\ \hline
adult & 32561 & 2 & 123\\ \hline
fourclass & 862 & 2 & 2\\ \hline
phishing & 11055 & 2 & 68\\ \hline
20news & 18040 & 20 & 50\\ \hline
satimage & 4435 & 6 & 36\\ \hline
letter & 20000 & 26 & 16\\ \hline
mnist & 70000 & 10 & 50\\ \hline
\end{tabular}
\end{center} 
\end{table}
We obtained six classification datasets from the UCI repository\footnote{http://archive.ics.uci.edu/ml/index.html} and also obtained 20newsgroups\footnote{http://qwone.com/~jason/20Newsgroups/} and MNIST datasets. 
We used the raw features for the datasets from the UCI repository.
For the 20newgroups dataset, we removed stop words and retained the 2000 most frequent words. We then removed documents with fewer than 10 words. We extracted tf-idf features and applied Principle Component Analysis (PCA) to reduce the dimensionality to 50. 
For the MNIST dataset, we applied PCA on the raw features to reduce the dimensionality to 50.
The dataset statistics are summarized in Table \ref{table:statistics}.

\section{Details of the Subcategory Shift Scenario} \label{app:subcategory}
In this section, we give details on how we converted the original multi-class classification problems into multi-class classification problems with fewer classes.

For the datasets from the UCI repository, we systematically grouped the class labels into binary categories by the following procedure. 
First, class labels are sorted by the number of data points in the classes. 
Then, 1, 3, 5, \ldots-th labels are assigned to a positive category and the others are assigned to a negative category. 
For MNIST, we considered a binary classification between odd and even numbers and set the original classes as subcategories. 
For 20newsgroups, we converted the original 20-class classification problem into a 7-class one with each class corresponding to a high-level topic: comp, rec, sci, misc, alt, soc, and talk. We then set the original classes as subcategories.

\section{Experimental Results measured by Surrogate Loss} \label{app:exp_surrogate}
In this section, we report the experimental results measured by the surrogate loss (the logistic loss).
We used the KL and the PE divergences, where we set $\delta = 0.5$.
We used the same $f$-divergence and the same $\delta$ during training and testing.
The experimental results using the KL and the PE divergences are reported in Tables \ref{table:surrogate_kl} and \ref{table:surrogate_pe}, respectively.
We empirically confirmed that \emph{in terms of the surrogate loss}, each method indeed achieved the best performance in terms of the metric it optimizes for.

\setlength\intextsep{0pt}
\setlength\textfloatsep{0pt}
\begin{table*}[t]
\tiny
\begin{center} 
\caption{Experimental comparisons of the three methods w.r.t.~the estimated ordinary risk and the estimated structural adversarial risk \emph{using the surrogate loss (the logistic loss)}. The lower these values are, the better the performance of the method is. The KL divergence is used and distribution shift is assumed to be (a) class prior change and (b) sub-category prior change. Mean and standard deviation over 50 random train-test splits were reported. The best method and comparable ones based on the t-test at the significance level 1\% are highlighted in boldface.} 
\label{table:surrogate_kl}
\subfloat[Class prior change.]{
\begin{tabular}{|l||c|c|c|c|c|c|c|c|c|c|}\hline
\raisebox{-1.3ex}{Dataset} & \multicolumn{3}{|c|}{\shortstack[c]{ Estimated ordinary risk}} &\multicolumn{3}{|c|}{\shortstack[c]{ Estimated adversarial risk}} &\multicolumn{3}{|c|}{\shortstack[c]{ Estimated structural adversarial risk}} \\ \cline{2-10}
             & \raisebox{0ex}{ERM} & \raisebox{0ex}{AERM} &  \shortstack[c]{Structural AERM} 
             & \raisebox{0ex}{ERM} & \raisebox{0ex}{AERM} &  \shortstack[c]{Structural AERM}
             & \raisebox{0ex}{ERM} & \raisebox{0ex}{AERM} &  \shortstack[c]{Structural AERM} \\ \hline
blood & {\bf 0.52 (0.05)} & 0.69 (0.0) & 0.62 (0.02) & 1.04 (0.1) & {\bf 0.69 (0.0)} & 0.97 (0.03) & 0.86 (0.23) & 0.69 (0.0) & {\bf 0.63 (0.19)}\\ \hline
adult & {\bf 0.33 (0.0)} & 0.65 (0.03) & 0.39 (0.0) & 1.28 (0.02) & {\bf 0.69 (0.01)} & 1.42 (0.01) & 0.59 (0.3) & 0.67 (0.01) & {\bf 0.4 (0.38)}\\ \hline
fourclass & {\bf 0.51 (0.05)} & 0.69 (0.0) & 0.54 (0.05) & 0.91 (0.04) & {\bf 0.69 (0.0)} & 0.88 (0.04) & 0.65 (0.13) & 0.69 (0.0) & {\bf 0.56 (0.13)}\\ \hline
phishing & {\bf 0.15 (0.01)} & 0.41 (0.08) & {\bf 0.15 (0.0)} & 0.86 (0.01) & {\bf 0.59 (0.08)} & 0.85 (0.01) & 0.18 (0.06) & 0.41 (0.02) & {\bf 0.16 (0.05)}\\ \hline
20news & {\bf 1.05 (0.01)} & 1.49 (0.04) & 1.22 (0.02) & 3.42 (0.02) & {\bf 3.0 (0.04)} & 3.58 (0.03) & 1.43 (0.1) & 1.74 (0.19) & {\bf 1.32 (0.13)}\\ \hline
satimage & {\bf 1.01 (0.01)} & 1.26 (0.02) & 1.29 (0.01) & 2.54 (0.02) & {\bf 2.15 (0.02)} & 2.86 (0.02) & 1.41 (0.05) & 1.59 (0.01) & {\bf 1.38 (0.04)}\\ \hline
letter & {\bf 0.37 (0.01)} & 0.47 (0.03) & 0.51 (0.02) & 1.65 (0.02) & {\bf 1.19 (0.03)} & 2.27 (0.03) & 0.77 (0.17) & 0.77 (0.09) & {\bf 0.58 (0.21)}\\ \hline
mnist & {\bf 0.35 (0.0)} & 0.59 (0.05) & 0.45 (0.0) & 1.96 (0.01) & {\bf 1.38 (0.04)} & 1.85 (0.01) & 0.49 (0.06) & 0.73 (0.02) & {\bf 0.47 (0.04)}\\ \hline
\end{tabular}}
\\
\subfloat[Sub-category prior change.]{
\begin{tabular}{|l||c|c|c|c|c|c|c|c|c|c|}\hline
\raisebox{-1.3ex}{Dataset} & \multicolumn{3}{|c|}{\shortstack[c]{ Estimated ordinary risk}} &\multicolumn{3}{|c|}{\shortstack[c]{ Estimated adversarial risk}} &\multicolumn{3}{|c|}{\shortstack[c]{ Estimated structural adversarial risk}} \\ \cline{2-10}
             & \raisebox{0ex}{ERM} & \raisebox{0ex}{AERM} &  \shortstack[c]{Structural AERM} 
             & \raisebox{0ex}{ERM} & \raisebox{0ex}{AERM} &  \shortstack[c]{Structural AERM}
             & \raisebox{0ex}{ERM} & \raisebox{0ex}{AERM} &  \shortstack[c]{Structural AERM} \\ \hline
20news & {\bf 0.61 (0.01)} & 0.84 (0.06) & 0.76 (0.05) & 2.76 (0.02) & {\bf 2.03 (0.05)} & 2.91 (0.03) & 1.02 (0.14) & 1.08 (0.29) & {\bf 0.89 (0.21)}\\ \hline
satimage & {\bf 0.63 (0.0)} & 0.69 (0.0) & 0.68 (0.0) & 0.99 (0.01) & {\bf 0.69 (0.0)} & 0.74 (0.0) & 0.81 (0.02) & 0.69 (0.0) & {\bf 0.69 (0.02)}\\ \hline
letter & {\bf 0.47 (0.0)} & 0.69 (0.0) & 0.64 (0.02) & 1.05 (0.03) & {\bf 0.69 (0.0)} & 0.86 (0.01) & 0.93 (0.02) & 0.69 (0.0) & {\bf 0.68 (0.06)}\\ \hline
mnist & {\bf 0.31 (0.0)} & 0.68 (0.01) & 0.39 (0.0) & 1.28 (0.01) & {\bf 0.69 (0.0)} & 1.25 (0.0) & 0.49 (0.05) & 0.68 (0.0) & {\bf 0.42 (0.03)}\\ \hline
\end{tabular}}
\end{center}
\end{table*}

\setlength\intextsep{0pt}
\setlength\textfloatsep{0pt}
\begin{table*}[t]
\tiny
\begin{center} 
\caption{Experimental comparisons of the three methods w.r.t.~the estimated ordinary risk and the estimated structural adversarial risk \emph{using the surrogate loss (the logistic loss)}. The lower these values are, the better the performance of the method is. The PE divergence is used and distribution shift is assumed to be (a) class prior change and (b) sub-category prior change. Mean and standard deviation over 50 random train-test splits were reported. The best method and comparable ones based on the t-test at the significance level 1\% are highlighted in boldface.} 
\label{table:surrogate_pe}
\subfloat[Class prior change.]{
\begin{tabular}{|l||c|c|c|c|c|c|c|c|c|c|}\hline
\raisebox{-1.3ex}{Dataset} & \multicolumn{3}{|c|}{\shortstack[c]{ Estimated ordinary risk}} &\multicolumn{3}{|c|}{\shortstack[c]{ Estimated adversarial risk}} &\multicolumn{3}{|c|}{\shortstack[c]{ Estimated structural adversarial risk}} \\ \cline{2-10}
             & \raisebox{0ex}{ERM} & \raisebox{0ex}{AERM} &  \shortstack[c]{Structural AERM} 
             & \raisebox{0ex}{ERM} & \raisebox{0ex}{AERM} &  \shortstack[c]{Structural AERM}
             & \raisebox{0ex}{ERM} & \raisebox{0ex}{AERM} &  \shortstack[c]{Structural AERM} \\ \hline
blood & {\bf 0.52 (0.05)} & 0.67 (0.02) & 0.61 (0.03) & 0.77 (0.04) & {\bf 0.69 (0.0)} & 0.81 (0.02) & 0.71 (0.07) & 0.69 (0.0) & {\bf 0.62 (0.07)}\\ \hline
adult & {\bf 0.33 (0.0)} & 0.41 (0.02) & 0.39 (0.01) & 0.69 (0.01) & {\bf 0.61 (0.01)} & 0.77 (0.01) & 0.49 (0.02) & 0.51 (0.0) & {\bf 0.4 (0.03)}\\ \hline
fourclass & {\bf 0.52 (0.05)} & 0.66 (0.02) & {\bf 0.53 (0.05)} & 0.73 (0.02) & {\bf 0.69 (0.01)} & 0.77 (0.04) & 0.6 (0.04) & 0.67 (0.0) & {\bf 0.54 (0.06)}\\ \hline
phishing & {\bf 0.15 (0.01)} & 0.2 (0.02) & {\bf 0.15 (0.0)} & 0.43 (0.01) & {\bf 0.4 (0.02)} & 0.43 (0.01) & 0.17 (0.02) & 0.21 (0.01) & {\bf 0.15 (0.01)}\\ \hline
20news & {\bf 1.04 (0.01)} & 1.15 (0.11) & 1.17 (0.02) & 1.99 (0.01) & {\bf 1.95 (0.1)} & 2.2 (0.02) & 1.29 (0.03) & 1.36 (0.06) & {\bf 1.24 (0.04)}\\ \hline
satimage & {\bf 1.01 (0.01)} & 1.1 (0.01) & 1.1 (0.01) & 1.81 (0.01) & {\bf 1.72 (0.02)} & 2.0 (0.01) & 1.26 (0.02) & 1.32 (0.01) & {\bf 1.17 (0.02)}\\ \hline
letter & {\bf 0.36 (0.01)} & 0.4 (0.01) & 0.42 (0.02) & 0.89 (0.02) & {\bf 0.81 (0.02)} & 1.0 (0.02) & 0.6 (0.04) & 0.59 (0.03) & {\bf 0.53 (0.05)}\\ \hline
mnist & {\bf 0.35 (0.0)} & 0.41 (0.01) & 0.43 (0.0) & 0.96 (0.01) & {\bf 0.87 (0.01)} & 1.04 (0.01) & 0.45 (0.02) & 0.5 (0.01) & {\bf 0.44 (0.01)}\\ \hline
\end{tabular}}
\\
\subfloat[Sub-category prior change.]{
\begin{tabular}{|l||c|c|c|c|c|c|c|c|c|c|}\hline
\raisebox{-1.3ex}{Dataset} & \multicolumn{3}{|c|}{\shortstack[c]{ Estimated ordinary risk}} &\multicolumn{3}{|c|}{\shortstack[c]{ Estimated adversarial risk}} &\multicolumn{3}{|c|}{\shortstack[c]{ Estimated structural adversarial risk}} \\ \cline{2-10}
             & \raisebox{0ex}{ERM} & \raisebox{0ex}{AERM} &  \shortstack[c]{Structural AERM} 
             & \raisebox{0ex}{ERM} & \raisebox{0ex}{AERM} &  \shortstack[c]{Structural AERM}
             & \raisebox{0ex}{ERM} & \raisebox{0ex}{AERM} &  \shortstack[c]{Structural AERM} \\ \hline
20news & {\bf 0.61 (0.01)} & 0.67 (0.01) & 0.68 (0.02) & 1.34 (0.01) & {\bf 1.22 (0.01)} & 1.4 (0.01) & 0.86 (0.04) & 0.86 (0.03) & {\bf 0.81 (0.05)}\\ \hline
satimage & {\bf 0.63 (0.0)} & 0.69 (0.0) & 0.68 (0.0) & 0.86 (0.01) & {\bf 0.69 (0.0)} & 0.73 (0.0) & 0.76 (0.01) & 0.69 (0.0) & {\bf 0.69 (0.01)}\\ \hline
letter & {\bf 0.47 (0.0)} & 0.67 (0.0) & 0.57 (0.04) & 0.79 (0.01) & {\bf 0.69 (0.0)} & 0.73 (0.01) & 0.74 (0.01) & 0.69 (0.0) & {\bf 0.66 (0.03)}\\ \hline
mnist & {\bf 0.31 (0.0)} & 0.4 (0.02) & 0.36 (0.0) & 0.72 (0.0) & {\bf 0.6 (0.01)} & 0.72 (0.0) & 0.44 (0.01) & 0.48 (0.0) & {\bf 0.41 (0.01)}\\ \hline
\end{tabular}}
\end{center}
\end{table*}

\section{Experimental Results with the PE divergence} \label{app:exp_pe}
In this section, we report the experimental results using the PE divergence, where we set $\delta = 0.5$.
The experimental results are reported in Table \ref{table:pe}.

\setlength\intextsep{0pt}
\setlength\textfloatsep{0pt}
\begin{table*}[t]
\small
\begin{center} 
\caption{Experimental comparisons of the three methods w.r.t.~the estimated ordinary risk and the estimated structural adversarial risk \emph{using the 0-1 loss (\%)}. The lower these values are, the better the performance of the method is. The PE divergence is used and distribution shift is assumed to be (a) class prior change and (b) sub-category prior change. Mean and standard deviation over 50 random train-test splits were reported. The best method and comparable ones based on the t-test at the significance level 1\% are highlighted in boldface.} 
\label{table:pe}
\subfloat[Class prior change.]{
\begin{tabular}{|l||c|c|c|c|c|c|c|}\hline
\raisebox{-1.3ex}{Dataset} & \multicolumn{3}{|c|}{\shortstack[c]{Estimated ordinary risk}} &\multicolumn{3}{|c|}{\shortstack[c]{Estimated structural adversarial risk}}  \\ \cline{2-7}
             & \raisebox{0ex}{ERM} & \raisebox{0ex}{AERM} &  \shortstack[c]{Structural AERM}
             & \raisebox{0ex}{ERM} & \raisebox{0ex}{AERM} &  \shortstack[c]{Structural AERM} \\ \hline
blood & {\bf 22.2 (0.6)} & {\bf 22.4 (0.5)} & 33.0 (2.0) & 47.4 (1.6) & 49.1 (1.1) & {\bf 36.6 (2.4)}\\ \hline
adult & {\bf 15.3 (0.1)} & {\bf 15.3 (0.1)} & 18.7 (0.2) & 24.9 (0.3) & 24.8 (0.3) & {\bf 19.1 (0.3)}\\ \hline
fourclass & {\bf 23.7 (1.2)} & {\bf 23.5 (1.2)} & 27.0 (1.3) & 32.7 (1.7) & 32.6 (1.8) & {\bf 28.7 (1.7)}\\ \hline
phishing & 6.0 (0.2) & 6.1 (0.2) & {\bf 5.9 (0.2)} & 7.1 (0.3) & 7.4 (0.3) & {\bf 6.4 (0.3)}\\ \hline
20news & {\bf 28.8 (0.3)} & 29.7 (0.3) & 33.7 (0.3) & 37.9 (0.3) & 38.4 (0.4) & {\bf 37.5 (0.4)}\\ \hline
satimage & {\bf 25.1 (0.2)} & 26.7 (0.3) & 28.1 (0.3) & 33.7 (0.4) & 35.6 (0.4) & {\bf 32.0 (0.4)}\\ \hline
letter & {\bf 14.2 (0.5)} & {\bf 14.5 (0.5)} & 15.5 (0.5) & 26.5 (0.9) & 25.6 (0.9) & {\bf 20.8 (0.7)}\\ \hline
mnist & {\bf 10.0 (0.1)} & 10.1 (0.1) & 12.2 (0.1) & 13.3 (0.1) & 13.2 (0.1) & {\bf 13.1 (0.1)}\\ \hline
\end{tabular}}
\\
\subfloat[Sub-category prior change.]{
\begin{tabular}{|l||c|c|c|c|c|c|c|}\hline
\raisebox{-1.3ex}{Dataset} & \multicolumn{3}{|c|}{\shortstack[c]{Estimated ordinary risk}} &\multicolumn{3}{|c|}{\shortstack[c]{Estimated structural adversarial risk}}  \\ \cline{2-7}
             & \raisebox{0ex}{ERM} & \raisebox{0ex}{AERM} &  \shortstack[c]{Structural AERM}
             & \raisebox{0ex}{ERM} & \raisebox{0ex}{AERM} &  \shortstack[c]{Structural AERM} \\ \hline
20news & {\bf 19.0 (0.3)} & 19.6 (0.4) & 20.8 (0.4) & 29.2 (0.4) & 29.7 (0.4) & {\bf 27.3 (0.4)}\\ \hline
satimage & {\bf 36.4 (0.3)} & 41.1 (1.9) & 39.6 (0.5) & 53.9 (0.4) & 56.7 (2.7) & {\bf 47.0 (0.5)}\\ \hline
letter & {\bf 17.4 (0.4)} & 18.5 (0.5) & 23.1 (3.3) & {\bf 38.0 (0.5)} & 39.5 (0.6) & 38.9 (1.0)\\ \hline
mnist & {\bf 13.3 (0.1)} & 13.5 (0.1) & 15.6 (0.2) & 20.0 (0.2) & 20.2 (0.2) & {\bf 18.6 (0.2)}\\ \hline
\end{tabular}}
\end{center}
\vspace{-0.6cm}
\end{table*}

\end{document}